\documentclass[journal]{IEEEtran}
\usepackage[T1]{fontenc}
\usepackage{lmodern}
\usepackage{amsmath}
\usepackage{amsfonts,amsthm}
\usepackage{algorithm}
\usepackage{array}
\usepackage[caption=false,font=normalsize,labelfont=sf,textfont=sf]{subfig}
\usepackage{textcomp}
\usepackage{url}
\usepackage{verbatim}
\usepackage{graphicx}
\usepackage{cite}
\usepackage{algpseudocode}
\usepackage{booktabs}
\usepackage{graphicx}

\usepackage{hyperref}

\usepackage{comment}
\usepackage{lmodern}
\usepackage{xcolor}
\usepackage[most]{tcolorbox}
\usepackage{makecell}
\usepackage{multirow}
\usepackage{bbm}

\definecolor{preferredblue}{RGB}{207, 226, 243}
\definecolor{rejectedpink}{RGB}{248, 215, 218}
\definecolor{outerbg}{RGB}{255, 255, 255}
\definecolor{darkblue}{rgb}{0, 0, 0.5} \hypersetup{colorlinks=true, citecolor=darkblue, linkcolor=darkblue, urlcolor=darkblue} 
\newtheorem{definition}{Definition}
\newtheorem{theorem}{Theorem}
\newtheorem{lemma}{Lemma}
\newtheorem{proposition}{Proposition}
\newtheorem{remark}{Remark}
\newtheorem{assumption}{Assumption}

\begin{document}

\title{KV Cache Compression Through the Lens of Transform Coding}

\author{Hannah Laus, Claudio Mayrink Verdun, Hao Wang, Flavio du Pin Calmon, Felix Krahmer
        % <-this % stops a space
\thanks{Hannah Laus and Felix Krahmer are with the Technical University of Darmstadt, the Technical University of Munich and the Munich Center for Machine Learning (\url{hannah.laus@tum.de})\\
Claudio Mayrink Verdun is with the Massachusetts Institute of Technology\\
Hao Wang is with the Red Hat AI Innovation Team and the MIT-IBM Watson AI Lab\\
Flavio du Pin Calmon is with the John A. Paulson School Of Engineering And Applied Sciences at Harvard University}% <-this % stops a space
\thanks{}}

\maketitle

\begin{abstract}

The key-value (KV) cache stores information from past tokens and is a major memory bottleneck in long-context inference. Existing quantization methods address this bottleneck by representing the KV cache uniformly with lower-precision data types and designing quantization schemes to minimize reconstruction error in the cache itself, without accounting for how that error propagates through attention mechanisms. We prove that, under a white-noise quantization model, the expected attention-aware distortion decomposes into additive key and value contributions that factor across tokens and channels. Building on transform coding and reverse water-filling, which are classical tools from signal processing and rate-distortion theory, we introduce Attention-Aware Transform Coding (AATC), which allocates bits over a calibration set to minimize attention-aware distortion. On \texttt{Llama-3.1-8B-Instruct} and \texttt{Qwen-2.5-7B-Instruct}, evaluated across \texttt{LongBench}, \texttt{RULER}, \texttt{GSM8K}, \texttt{MMLU-Pro}, and \texttt{MATH-500}, our method achieves near-lossless accuracy at approximately $5.8\times$ compression, whereas each baseline degrades in at least some settings.
\end{abstract}


\section{Introduction}
Signal quantization, that is, the representation of signals by bit sequences, is a key building block of digital signal processing. A central task in this context is bit allocation, the question of how many bits to use for different parts of the signal representation. This topic is well studied in the signal processing literature and in various models, one can precisely characterize how to optimally allocate the bits (see \cite{gersho2012vector} for a detailed discussion). 

Key ingredients for many methods attempting optimal bit allocation are transform coding \cite{goyal2001theoretical}  and reverse waterfilling \cite{cover1999elements}. The transform coding step whitens the signal, yielding a signal representation with components that are uncorrelated but of varying strengths. The reverse waterfilling step then distributes the bits to best represent the components under a fixed budget, allocating more bits to larger components and fewer bits to smaller components.

In recent years, an analogous problem of data representation with a fixed bit budget has arisen in machine learning in the context of inference with large language models (LLMs). At inference time, an LLM decomposes its input into tokens and then generates new tokens by sampling from next token predictions. In transformer models \cite{vaswani2017attention}, each token yields a sequence of key-value (KV) pairs at each layer of the model. These values are cached to accelerate inference when generating new tokens.  
At the same time, LLMs increasingly operate over long contexts: documents, transcripts, and multi-turn histories spanning more than a million tokens. As a consequence, the KV cache grows linearly with context length, and at long contexts it becomes the dominant memory cost of inference. 
KV cache compression has become essential for deploying transformer-based LLMs under realistic memory and latency budgets. In that context, efficient quantization plays a key role \cite{liu2024kivi, hooper2024kvquant}. Approaches based on low-rank projection \cite{chang2024palu,staniszewski2025kv} can be interpreted as identifying components with no bits allocated to them. 
A third ingredient in state-of-the-art KV cache compression that is somewhat orthogonal is based on eviction \cite{zhang2023h2o,xiao2023efficient}.

Quantization and low-rank projection are highly effective for KV cache compression, and several such methods already draw on ideas from signal quantization and transform coding \cite{staniszewski2025kv,chang2024palu,zuo2026ratequant}. We develop this connection systematically, studying KV cache compression through the lens of transform coding. 
This perspective leads us to an attention-aware distortion measure that, rather than minimizing raw reconstruction error on the cache, decomposes the impact of quantization on the output into three distinct error dimensions, which have been individually studied in previous works, but without including their interaction. The decomposition unifies several previously disconnected viewpoints: it shows that these separately-discovered effects in token eviction, quantization or low-rank projection are facets of a single distortion objective, which both motivates our method and provides a unifying perspective on existing work that we map onto the individual factors in Section~\ref{sec:implications}.

To distribute the bits in a way that leads to a small distortion, we determine a bit allocation scheme via a reverse waterfilling algorithm~\cite{cover1999elements} applied to calibration data and then apply it individually to each token.
Inspired by the fact that reverse waterfilling is designed for uncorrelated data and is therefore often combined with a transform coding step for decorrelation~\cite{goyal2001theoretical}, our method first whitens the key and value features on a calibration set, which approximately decorrelates them and transforms the data into components of varying strength, before distributing the total bit budget across the transformed channels by reverse waterfilling. We validate these findings on \texttt{Llama-3.1-8B-Instruct} and \texttt{Qwen-2.5-7B-Instruct}, across \texttt{LongBench}, \texttt{RULER}, \texttt{GSM8K}, \texttt{MMLU-Pro}, and \texttt{MATH-500} which are standard long-context quality benchmarks. 

Our main contributions are as follows.
\begin{itemize}
    \item We derive an attention-aware distortion measure that separates key and value quantization errors into token- and channel-dependent factors.
    \item We provide a unifying perspective across several existing KV cache compression methods by relating them to the attention-aware distortion measure.
    \item Based on this measure, we develop Attention-Aware Transform Coding (AATC), which combines whitening with reverse waterfilling to obtain an optimal channel-wise allocation under the attention-aware distortion.
    \item We show that AATC achieves near-lossless accuracy across all benchmarks at approximately $5.8\times$ compression, outperforming the baselines (KIVI, KVQuant and PALU), since they degrade at least on one benchmark.
\end{itemize}

\section{Related Work}
\label{sec:related-work}
\subsection{KV cache Compression}
Approaches to reducing the memory footprint of the KV cache fall into several families. Scalar quantization methods reduce per-element precision: KIVI \cite{liu2024kivi} uses per-channel asymmetric quantization for keys and per-token quantization for values; KVQuant \cite{hooper2024kvquant} employs non-uniform quantization with sensitivity-weighted codebooks and dense outlier handling.
  These methods preserve all dimensions of the cache and degrade precision uniformly or near-uniformly. Vector quantization methods such as A²ATS \cite{he2025a2ats}, CommVQ \cite{li2025commvq}, and TurboQuant \cite{zandieh2025turboquant} jointly quantize groups of coefficients via learned codebooks. 
  Low-rank and transform-based methods, such as PALU \cite{chang2024palu}, Eigen Attention \cite{saxena2024eigen} and xKV \cite{chang2025xkv} compress along the feature dimension by deleting low-energy subspaces from the cache. Transform-coding methods instead decorrelate and allocate bits without hard truncation: KVTC \cite{staniszewski2025kv} applies a global, cross-layer transform folowed by bit allocation and is the closest prior method to ours.  MiniCache \cite{liu2024minicache} alternatively merges KV pairs across layers. 
Rotation-based methods \cite{su2025rotatekv},\cite{zhou2026oscar} apply orthogonal transforms before quantization to suppress outliers.
A complementary line of work \cite{wang2025squat,zhang2026mixkvq} aims to minimize the distortion after multiplying key and query vectors, as this corresponds to the input of the attention.

Eviction and merging methods reduce the cache by combining or deleting previous tokens: H2O \cite{zhang2023h2o}, SnapKV \cite{li2024snapkv} and StreamingLLM \cite{xiao2023efficient} evict tokens based on different importance scoring techniques, which are all related to attention-derived importance.
A more recent line of eviction work \cite{guo2024attention,feng2025identify,goel2025caote,devoto2025expected} refines the importance criterion beyond accumulated attention scores by incorporating value-side information. The new eviction criteria mentioned in those works correspond to projections of our decomposition onto its token-dependent factor (Section~\ref{sec:implications}).
Our work is in the transform-coding family and proposes a bit allocation criterion derived from a general distortion analysis. We perform empirical comparisons against KIVI, PALU, and KVQuant in Section \ref{sec:experiments}. In Appendix~\ref{app:mapping-distortion} we relate many of the above works to our distortion decomposition (Theorem~\ref{thm:distortion}).

\subsection{Theoretical Analyses of KV cache Quantization}
Several recent works analyze KV cache quantization theoretically, each tied to a specific method.  Closest to our key-side analysis is QAQ \cite{dong2024qaq}, which via a short second-order expansion obtains the query-weighted key sensitivity 
and allocates per-token bits from the aggregate query norm. We instead resolve the query weighting per channel, carry a controlled remainder (Lemma \ref{lem:softmax-tail}), and add the value-side weight. 
Concurrent work \cite{calver2026runtime} likewise splits the attention-output error into key and value terms, but as a worst-case runtime bound used to fall back to exact attention rather than to allocate precision offline, so it carries no channel weights or transform. 
~\cite{su2026eoptshrinkq} applies random matrix theory, modeling the cache as a low-rank shared-context component plus a residual that is unique for every token. 
RateQuant~\cite{zuo2026ratequant} is based on applying rate-distortion theory to head-level mixed-precision allocation.
Finally, Zhou et al. ~\cite{zhou2026oscar} prove that attention-aware rotations are optimal under a frozen residual model, with the per-channel bit allocation held uniform rather than optimized. We further relate the last two papers to our distortion decomposition in Section~\ref{sec:implications} and Appendix \ref{app:mapping-distortion}.

\subsection{Signal Quantization}
On the signal-processing side, our method instantiates the classical transform-coding architecture, first decorrelate, then allocate, introduced by Huang and Schultheiss \cite{huang1963block}, whose components we review in Section \ref{subsec:transform-coding}. Closest in spirit is perceptual coding \cite{jayant1993signal}, which adapts the distortion measure to the receiver of the signal, there the human perceptual system. We adapt it to the transformer's read-out. Here attention plays the role of the perceptual model, with the induced weighted metric derived in Theorem \ref{thm:distortion} rather than posited. Learned nonlinear transform coding \cite{balle2020nonlinear} continues the same blueprint with trained nonlinear transforms. In contrast, ours is fixed by calibration statistics, so the transform is exactly invertible and any loss is confined to the allocation step.

\section{Preliminaries}\label{sec:preliminaries}
We briefly review transform coding, attention and KV cache quantization. 
\subsection{Transform Coding}
\label{subsec:transform-coding}
Let $\mathbf{x} \in \mathbb{R}^{d}$ be a zero-mean source with
covariance $\boldsymbol{\Sigma} \in \mathbb{R}^d \times \mathbb{R}^d$, and let $\hat{\mathbf{x}}$ be
its reconstruction after lossy compression. We measure fidelity
by the average squared-error distortion per dimension \cite{gersho2012vector},
\begin{equation}\label{eq:avg-distortion}
    D
    = \frac{1}{d}\,
      \mathbb{E}\!\left[
        \lVert \mathbf{x} - \hat{\mathbf{x}} \rVert_2^2
      \right]
    = \frac{1}{d} \sum_{i=1}^{d} D_i,
    \qquad
    D_i = \mathbb{E}\!\left[(x_i - \hat{x}_i)^2\right].
\end{equation}
When each coordinate is quantized separately, the design
problem is thus to distribute a bit budget across coordinates.
 
\emph{Transform coding} \cite{goyal2001theoretical}, the
framework behind classical media codecs such as JPEG
\cite{wallace1991jpeg}, solves this problem in a transformed
coordinate system: an orthonormal transform
$\mathbf{y} = \mathbf{U}^{\top}\mathbf{x}$ decorrelates the
source and concentrates its variance in a few leading
coefficients (\emph{energy compaction}); the coefficients are
quantized with independent scalar quantizers, and the signal is
reconstructed as $\hat{\mathbf{x}} = \mathbf{U}\hat{\mathbf{y}}$.
Orthonormal transforms preserve the MSE, so
\eqref{eq:avg-distortion} applies unchanged to the coefficients
$y_i$, whose variances $\sigma_i^2$ now differ widely:
high-variance coefficients deserve many bits, low-variance ones
few or none. For Gaussian sources, the optimal transform is the
Karhunen--Lo\`{e}ve transform, the eigenbasis of
$\boldsymbol{\Sigma}$. In practice one uses a whitening transform estimated from data.
 
To allocate the bits, note that under high-rate assumptions a
scalar quantizer with $b_i$ bits achieves
$D_i \approx c\, \sigma_i^2\, 2^{-2b_i}$ for a
quantizer-dependent constant $c$ \cite{gersho2012vector}.
Minimizing \eqref{eq:avg-distortion} subject to
$\frac{1}{d}\sum_i b_i = \bar{b}$ and $b_i \ge 0$ yields the
classical allocation of Huang and Schultheiss \cite{huang1963block}, which coincides with the reverse waterfilling solution for a source of independent Gaussian components \cite[Chapter~10]{cover1999elements}:
\begin{align}
b_i^{\star}
    = 
\begin{cases}
     \frac{1}{2}\log_2
      \frac{\sigma_i^2}
           {\lambda},
   &  \text{ if } \sigma_i^2 > \lambda\\
    0 & \text{ otherwise,}
    \label{eq:bit-allocation}
\end{cases}
\end{align}
where the water level $\lambda$ is chosen so that $\frac{1}{d}\sum_i b_i^{\star} = \bar{b}$.
Every coefficient which is retained is quantized to the same distortion $c \lambda$, while coefficients whose variance falls below the water level receive no bits and are reconstructed as zero. 

Both the transform and the allocation
depend on source statistics that are typically unknown; in
practice, they are estimated from a calibration sample, as we do
in this work.
\subsection{Attention and the KV cache}\label{subsec:attention}
A decoder-only transformer \cite{radford2018improving}, the architecture behind modern large language models, takes a sequence of tokens and predicts the next token. Internally, each token is represented as a vector and the transformer applies $L$ layers in sequence, to map the current token vectors to new ones. Within a layer, the token vectors are first mixed across the sequence and then passed through a per-token nonlinear-map. The mixing step, called self-attention, is the only operation that introduces interdependence between the tokens.

In each layer the input is linearly projected into three sequences of vectors, the queries, keys and values. Intuitively, a token's query states what information it is looking for, its key states what it offers and its value is the content it contributes. The self-attention mechanism maps each query
to a data-dependent weighted combination across tokens. Unlike a fixed linear filter, the combination weights are computed from the input itself. More precisely, for every query, self-attention forms a scaled dot-product similarity with all keys, passes these scores through a softmax to obtain nonnegative weights that sum to one, and returns the corresponding weighted average of the values. 

To generate text, the model predicts the next token, appends it to the sequence and repeats. The sequence is the given context followed by the tokens generated so far, so its length $T$ grows by one at every step and with it the set of keys and values that must remain available. To avoid recomputing them, the keys and values of all past tokens are stored in the KV cache, whose size therefore grows with $T$.

To represent this in mathematical notation,
let $T$ be the number of tokens and $L$ the number of layers. The input to layer $\ell$ is $\mathbf{X}^{(\ell)} \in \mathbb{R}^{T \times d_{\mathrm{model}}}$ whose $t$-th row $x_t^{(\ell)} \in \mathbb{R}^{d_{\mathrm{model}}}$ is the hidden representation of token $t$ at layer $\ell$. Suppressing $\ell$ the query, key and value sequences introduced above are  $\mathbf{Q} = \mathbf{X}\mathbf{W}_Q$, $\mathbf{K} = \mathbf{X}\mathbf{W}_K$, $\mathbf{V} = \mathbf{X}\mathbf{W}_V$ with  $\mathbf{W}_Q^{(\ell)}, \mathbf{W}_K^{(\ell)} \in \mathbb{R}^{d_{\mathrm{model}} \times d_k}$ and $\mathbf{W}_V^{(\ell)} \in \mathbb{R}^{d_{\mathrm{model}} \times d_v}$ (typically $d_k = d_v = d_{\mathrm{model}}$) and rows $q_t=x_t\mathbf{W}_Q , k_t, v_t$ .
\begin{definition}[Scaled Dot-Product Attention \cite{vaswani2017attention}]\label{def:sdpa}
Given $\mathbf{Q}, \mathbf{K} \in \mathbb{R}^{T \times d_k}$, $\mathbf{V} \in \mathbb{R}^{T \times d_v}$, and a mask $\mathbf{M}\in (\mathbb{R} \cup \{-\infty\})^{T \times T}$,
\begin{align*}
\mathsf{Attention}(\mathbf{Q}, \mathbf{K}, \mathbf{V}) &:= \mathbf{A} \mathbf{V} \mathbf{W}_O, \; \mathbf{A} = \mathsf{softmax}\!\left(\tfrac{\mathbf{Q} \mathbf{K}^\top}{\sqrt{d_k}} + \mathbf{M}\right), \end{align*}
with row-wise softmax weights
\begin{align*}
 a_{ti} &= \frac{\exp(s_{ti})}{\sum_{j} \exp(s_{tj})}, \quad s_{ti} = \frac{q_t^\top k_i}{\sqrt{d_k}} + M_{ti},
\end{align*}
so $a_{ti} \geq 0$ and $\sum_i a_{ti} = 1$.
Throughout the remainder of the paper, we consider the following causal mask: $M_{ti} = 0$ for $i \leq t$ and $-\infty$ otherwise. Further, we consider mostly a single decode-step attention, which corresponds to the current row of attention $\mathsf{Attention}(\mathbf{Q}, \mathbf{K}, \mathbf{V})$ as introduced above, i.e.,
\begin{align}
\label{eq:singleheadattoutput}
o_t = \sum_{i=1}^{t} a_{ti} \, v_i \, \mathbf{W}_O.
\end{align}
\end{definition}
\begin{remark}
Modern LLMs adopt different variants of attention mechanisms (e.g., group query attention \cite{ainslie2023gqa}). Our method, described later, naturally extends to them (see Section \ref{sec:method-channel-allocation}), but for now we will report everything as defined above.
\end{remark}
LLM inference has two phases. In the \emph{prefill} phase, a prompt of $T$ tokens is processed in parallel and for each layer $\ell$, the keys $\mathbf{K}^{(\ell)} \in \mathbb{R}^{T \times d_k}$ and values $\mathbf{V}^{(\ell)} \in \mathbb{R}^{T \times d_v}$ are stored in the cache. In the \emph{decode} phase, each step $t = T+1, T+2, \dots$ processes only the new token, appends its key $k_t$ and value $v_t$, and computes attention for $q_t$ against all $T$ cached keys and values. The KV cache at step $t$ is
\begin{equation*}
\mathcal{C}_t = \bigl\{ (\mathbf{K}^{(\ell)}_{\leq t}, \mathbf{V}^{(\ell)}_{\leq t}) \mid \ell \in [L] \bigr\},
\end{equation*}
where $\mathbf{K}^{(\ell)}_{\leq t} \in \mathbb{R}^{t \times d_k}$ has rows $k_1, \dots, k_t$ and its columns are called channels; analogously for $\mathbf{V}^{(\ell)}_{\leq t}$. 

\subsection{Quantization of the KV cache}\label{subsec:quantization}
\begin{definition}[Uniform Scalar Quantization (following \cite{nagel2021white})]\label{def:uniform-quant}
A $b$-bit uniform scalar quantizer with scale $\Delta > 0$ and zero-point $z \in \mathbb{R}$ maps a scalar $x \in \mathbb{R}$ to a quantized integer and back via
\begin{align*}
\mathsf{qtz}(x) &= \mathsf{clamp}\!\left(\left\lfloor \frac{x - z}{\Delta} + \frac{1}{2} \right\rfloor, 0, 2^b - 1\right), \\
\mathsf{dqtz}(c) &= \Delta \cdot c + z, 
\end{align*}
where $\lfloor \cdot \rfloor$ denotes the floor function and $\mathsf{clamp}(x, a, b) = \min(\max(x, a), b)$ clips the range to $[a,b]$. The reconstruction $\hat{x} = \mathsf{dqtz}(\mathsf{qtz}(x))$ satisfies $|x - \hat{x}| \leq \Delta/2$ for values within the clipping range.
\end{definition}
A quantized cache replaces each $k_i, v_i$ by $\hat{k}_i = \mathsf{dqtz}(\mathsf{qtz}(k_i))$, $\hat{v}_i = \mathsf{dqtz}(\mathsf{qtz}(v_i))$, with error vectors $\delta k_i =\hat{k}_i - k_i $, $\delta v_i =\hat{v}_i - v_i $ The quantized output is $\hat{o}_t = \mathsf{Attention}(\tilde{q}_t, \hat{\mathbf{K}}_{\leq t}, \hat{\mathbf{V}}_{\leq t}) \mathbf{W}_O$. The pair $(\Delta, z)$ may be computed at different \emph{granularities}: per-tensor (one pair per matrix); per-token/row-wise (one pair per $k_i$); per-channel/column-wise (one pair per dimension $c$, adapting to each channel's range across tokens); or group quantization (one pair per contiguous group of $G$ elements along a chosen axis).
Transform-coding approaches such as KVTC~\cite{staniszewski2025kv} choose the quantizer to minimize the distortions $ \mathbb{E} \|k_i - \hat k_i\|^2$ and $ \mathbb{E} \|v_i - \hat v_i\|^2$ which only optimize over the reconstruction error of the keys and values, but not the reconstruction error of the attention output. Therefore, these are only proxies for
the quantity of interest, a gap that
\cite{staniszewski2025kv} themselves acknowledge. The quantity of interest is the output distortion defined as follows.
\begin{definition}[Attention output distortion]
\label{def:attention-output-distortion}
In line with the discussion in Section \ref{subsec:transform-coding},
the expected squared-error distortion between the quantized attention output $\hat o_t$ and $o_t$, the attention output for token $t$, is defined as
    \[
    D = \mathbb{E} \|o_t - \hat o_t \|^2_2 ,
    \]
    averaged over the empirical distribution of keys and values.
\end{definition} 
Three effects drive the gap between these two distortion measures. First, the keys influence the attention output only through their alignment with the query. Second, a token with almost no attention weight $a_i$ can be quantized coarsely for both its keys and values without increasing the output distortion. Third, value errors enter the attention output linearly and are reshaped by the output projection. 
Theorem~\ref{thm:distortion}
characterizes this relationship explicitly, and Section~\ref{sec:implications}
shows how the resulting decomposition subsumes several asymmetries that prior methods exploit only in isolation.

\section{Attention-Aware Distortion Formula for KV cache compression}\label{sec:method}
In this section, we decompose the difference between the outputs of the attention mechanism computed using the quantized and original KV caches. 
Theorem~\ref{thm:distortion} decomposes the distortion, defined in Definition \ref{def:attention-output-distortion}, into key- and
value-induced contributions in which these factors appear as explicit
weights, yielding the distortion measure our bit allocation minimizes.
\subsection{Derivation of the Distortion Formula}
To make the distortion analytically tractable, we model the quantization noise under the
standard additive white-noise hypothesis stated below and we assume the key quantization error to be bounded.

\begin{assumption}[White-noise quantization model]\label{ass:white-noise}
The quantization errors $\delta\textnormal{k}_{i c}=\hat{ \textnormal{k}}_{i c} - k_{i c}$, $\delta\textnormal{v}_{i c}= \hat{\textnormal{v}}_{i c}- v_{i c}$ ($k_{i c}$ is the $c$-th element of the vector  $k_{i}$, analogue for the values) are zero-mean, mutually independent
across tokens $i$ and channels $c$, independent between keys and values, and symmetrically
distributed. In particular, all odd moments vanish,
$\mathbb{E}[\delta\textnormal{k}_{i c}^{\,3}] = \mathbb{E}[\delta\textnormal{v}_{i c}^{\,3}] = 0$.
\end{assumption}
\begin{assumption}[Bounded key quantization error]\label{ass:bounded}
There exists a constant $\kappa \ge 1$ such that, almost surely,
\begin{equation*}
|\delta\textnormal{k}_{i c}| \;\le\; \kappa\,\sigma_K
\qquad \text{for all tokens } i \text{ and channels } c .
\end{equation*}
\end{assumption}

\begin{remark}\label{rem:bounded} 
The additive white-noise model of quantization (Assumption \ref{ass:white-noise}) is the standard idealization underlying rate-distortion analyses of quantization
\cite{bennett1948spectra, widrow1996statistical, gray1998quantization}, even though it is known to be exact only in the high-rate limit, we demonstrate in our experiments that the distortion metric remains predictive at $2-4$ bits.
We adopt it in this paper, as
for the uniform symmetric quantizer used in Section \ref{sec:method-channel-allocation}, the
zero-mean and symmetry properties are in exact analogy to the transform coding literature. Key and value independence follow similarly, since they arise from distinct projections. Cross-channel independence is justified along similar lines when the channels are decorrelated before quantization. 
In contrast, cross-token independence is an idealization: neighboring KV
representations are correlated, so per-token errors are not strictly independent. Exploiting this dependence, however, would require a computationally intractable token varying bit allocation.
For that reason, almost all quantization methods rely on an idealization along these lines. Accordingly, we adopt the standard independence approximation used in transform coding and most existing KV cache quantization methods.

Assumption \ref{ass:bounded} corresponds to boundedness of the token entries as the max-based token scaling (Section~\ref{sec:method-channel-allocation}) is non-overloading. Namely, the quantization step scales with the maximum entry and each key error is at most half a quantization step (Definition \ref{def:uniform-quant}).
\end{remark}

The following theorem shows that under this model, the distortion, defined in \ref{def:attention-output-distortion}, decomposes into a key term
and a value term, each factoring across tokens and channels.
\begin{theorem}[Attention Output Distortion]\label{thm:distortion} Consider the single decode-step attention model of \eqref{eq:singleheadattoutput} and assume that its quantized keys $\hat{\textnormal{k}}_i = k_i+\delta \textnormal{k}_i$ and its quantized values $\hat{\textnormal{v}}_i = v_i + \delta\textnormal{v}_i$ satisfy the white-noise hypothesis (Assumption \ref{ass:white-noise}) and assume that $\delta\textnormal{k}_i$ is bounded (Assumption \ref{ass:bounded}). Then the expected squared distortion of the attention output $D$ separates
into additive key and value contributions,
\begin{equation}\label{eq:distortion-decomposition}
    D \;=\; \mathbb{E}\bigl[\lVert o - \hat{\textnormal{o}} \rVert^2\bigr]
    \;=\; D_K + D_V + R,
\end{equation}
\begin{align*}
    D_K &= \sum_{i,c} a_i^2 \,\lVert v_i \mathbf{W}_O - o \rVert^2 \, q_c^2 \,\mathbb{E}\bigl[(\delta\textnormal{k}_{i c})^2\bigr],\\
    D_V &= \sum_{i,c} a_i^2 \,\mathbb{E}\bigl[(\delta\textnormal{v}_{i c})^2\bigr]\, \lVert \mathbf{W}_{O_c} \rVert_F^2,
\end{align*}
each factoring across tokens $i$ and channels $c$, dropping the token index $t$ in $o_t$ for readability. The remainder is higher order in $\sigma_K$ and $\sigma_V$,
\begin{align*}
|R| &\le C\,\|q\|^2\|\mathbf{W}_O\|_F^2\,\sigma_K^2\sigma_V^2
+ C\,\bigl(A_0^3\,\|\mathbf{W}_O\|_F^2\,\sigma_V^2 + D_o\, A_0^4\bigr),
\end{align*}
where $A_0 = \kappa\|q\|_1\sigma_K$, $D_o := \max_i\|v_i\mathbf{W}_O - o\|^2$ and $C$ is absolute constant. $R$ is negligible whenever the
effective noise scales $A_0$ and $\|\mathbf{W}_O\|_F^2\sigma_V^2$ are small.
\end{theorem}
\begin{remark}[Remainder term]\label{rem:bits}
The remainder measures the difference between the exact output distortion and its second-order approximation. The theoretical derivation guarantees that it vanishes asymptotically as the quantization noise decreases. In practice, we work in low-bit regimes, the experiments in Section \ref{sec:experiments} suggest that the approximation even remains predictive in those regimes.
\end{remark}
We will continue with a proof sketch of Theorem \ref{thm:distortion}, the full proof is given in Appendix \ref{sec:softmax-derivatives}.

\begin{proof}[Proof Sketch Theorem~\ref{thm:distortion}]
In short, we expand the output error to second order in the quantization noise, take expectations, keep the leading terms, and bound the rest.

First, subtracting the quantized output from the exact one cancels the unperturbed term and splits the error into three parts, cancels against the corresponding term in $o$:
\begin{align*}
o - \hat{\textnormal{o}}
&= -\Bigl( \underbrace{\sum_{i=1}^{T} \delta\textnormal{a}_i v_i \mathbf{W}_O}_{=:\, E_K}
+ \underbrace{\sum_{i=1}^{T} a_i \delta\textnormal{v}_i \mathbf{W}_O}_{=:\, E_V}
+ \underbrace{\sum_{i=1}^{T} \delta\textnormal{a}_i \delta\textnormal{v}_i \mathbf{W}_O}_{=:\, E_{\times}} \Bigr),
\end{align*}
where $E_K$ collects the effect of the key error through the perturbed attention weights $\delta \textnormal{a}_i$, $E_V$ the effect of the value error, and $E_{\times}$ a second-order cross-term. A first- and second-order Taylor expansion of the softmax writes the weight perturbation as $\delta\textnormal{a}_i =   \delta\textnormal{a}_i^{(1)} + \delta\textnormal{a}_i^{(2)} + \textnormal{r}_i$, with the remainder $r_i $ controlled by Lemma~\ref{lem:softmax-tail}.

Second, expanding $\mathbb{E}\bigl[\|E_K + E_V + E_\times\|^2\bigr]$ and using that the key and value errors are independent and zero-mean (Assumption~\ref{ass:white-noise}), the key-value cross-terms $\mathbb{E}[ E_K E_V^\top]$ and $\mathbb{E}[E_K E_\times^\top]$ vanish exactly, and the remaining cross-terms are higher order. The distortion reduces to $\mathbb{E}[\|E_K\|^2 ] + \mathbb{E}[\|E_V\|^2]$ plus a remainder.

Third, for the key term, recentering each value by the output $o$ leaves $E_K$ unchanged, since the weight perturbations sum to zero and channel independence then gives $\mathbb{E}[E_K E_\times^\top] = \sum_{i,c} a_i^2 \,\lVert v_i \mathbf{W}_O - o \rVert^2 \, q_c^2 \,\mathbb{E}\bigl[(\delta\textnormal{k}_{i c})^2\bigr] = D_K$ up to higher-order terms.

Finally, the value term is exact. With $E_V = \sum_{i} a_i^2 (\delta\textnormal{v}_{i})^2 \mathbf{W}_{O} $ and deterministic $a_i$, token independence and zero mean leave only the diagonal, giving $D_V = \sum_{i,c} a_i^2 \,\mathbb{E}\bigl[(\delta\textnormal{v}_{i c})^2\bigr]\, \lVert \mathbf{W}_{O_c} \rVert_F^2$. 

Collecting the two leading terms and the bounded higher-order contributions yields $D = D_K + D_V + R$, with $R$ controlled through $A_0 = \kappa\|q\|_1\sigma_K$.
\renewcommand{\qedsymbol}{} 
\end{proof}

\subsection{Interpretation of the Distortion Measure}
\label{sec:implications}
\paragraph{Key distortion $D_K$} Each channel $c$ of the key token is weighted by (a) the squared attention weight $a_i^2$
, reflecting how much the token participates in the output; (b) the residual $\lVert v_i \mathbf{W}_O - o\rVert^2$, measuring how much switching attention to this token would change the output; and (c) the squared query component $q_c^2$, capturing which key channels the current query actually reads. Therefore, compressing channels for which $q_c^2$ is small or deleting tokens for which $a_i^2$ or $\lVert v_i \mathbf{W}_O - o\rVert^2$ is small does not change the distortion much.
\paragraph{Value distortion $D_V$} Each channel $c$ of value token $i$ is weighted by $a_i^2$, so only tokens that receive significant attention contribute to value distortion. The output projection $\mathbf{W}_{O_c}$ further scales the contribution of each value channel.\\
A central observation is that the distortion in~\eqref{eq:distortion-decomposition} factorizes: the token-dependent weights ($a_i^2$, $\lVert v_i \mathbf{W}_O - o\rVert^2$) and the channel-dependent weights ($q_c^2$, $\mathbf{W}_{O_c}^2$) appear in distinct factors.
Under Assumption \ref{ass:white-noise},  $\mathbb{E}[(\delta\textnormal{k}_{i c})^2]$ only
 depends on the bit allocation $b_c$
 for channel $c$ but is the same across tokens. Therefore, the total distortion factors are
 \begin{equation}\label{eq:factorized-distortion}
     D_K = \left(\sum_i a_i^2 \lVert v_i \mathbf{W}_O - o\rVert^2\right) \cdot \left(\sum_c q_c^2 \;\mathbb{E}\bigl[(\delta\textnormal{k}_{i c})^2\bigr]\right),
 \end{equation}
 and analogously for $D_V$. This separation has two consequences. First, it implies that we can focus on either channel-wise or token-wise compression and design a quantizer that allocates bits
accordingly. Second, the keys and values can be quantized independently:
 $D_K$ and $D_V$ are separate terms, with no cross-term coupling key and value quantization errors.
 Existing KV cache compression methods can be classified by which of the factors above they target. Table~\ref{tab:decomposition} summarizes this classification. A detailed explanation that identifies the methods addressing each factor and describes our method's position relative to them can be found in Appendix \ref{app:mapping-distortion}. Our method synthesizes the channel-level factors of the decomposition; the token-wise allocation is out of scope here.
\begin{table*}[t]
\centering
\caption{Distortion factors in $D = D_K + D_V$ and the design axes they motivate. Each row corresponds to one factor in the per-token, per-channel decomposition of attention-output error under KV quantization. Methods are listed by the factor they primarily target. Here, VQ$=$Vector quantization, MP $=$ mixed precision and TC $=$ transform coding. }
\label{tab:decomposition}
\small
\renewcommand{\arraystretch}{2.5}
\setlength{\tabcolsep}{4pt}
\begin{tabular}{@{}llll@{}}
\toprule
\textbf{Factor} & \textbf{Design axis} & \textbf{Mechanism} & \textbf{Representative methods}  \\
\midrule
\makecell[l]{$a_i^2$\\(token)} & \makecell[l]{token eviction} & \makecell[l]{drop tokens} & \makecell[l]{H2O \cite{zhang2023h2o}, SnapKV \cite{li2024snapkv}, StreamingLLM \cite{xiao2023efficient},\\ VATP \cite{guo2024attention}, CriticalKV \cite{feng2025identify}, Expected Attention \cite{devoto2025expected}}  \\
\makecell[l]{$\|v_i \mathbf{W}_O - o\|^2$\\(token)} & \makecell[l]{output-relevance} & \makecell[l]{drop tokens} & \makecell[l]{CAOTE \cite{goel2025caote}(pre-$\mathbf{W}_O$)}  \\
\makecell[l]{$q_c^2$, $C_Q$\\(K channel)} & \makecell[l]{query-aware\\key precision} & \makecell[l]{VQ/MP/rotation/ \\ TC+bit-alloc.} & \makecell[l]{A\textsuperscript{2}ATS \cite{he2025a2ats}, MixKVQ \cite{zhang2026mixkvq},\\OSCAR \cite{zhou2026oscar}, SQuat \cite{wang2025squat}, \textbf{Ours (Section\ref{sec:method-channel-allocation})}} \\
\makecell[l]{$\mathbf{W}_{O_c}$\\(V channel)} & \makecell[l]{output-projection-\\aware value} & TC+bit-alloc & (implicit in KVQuant \cite{hooper2024kvquant}), \textbf{Ours (Section\ref{sec:method-channel-allocation})} \\
\makecell[l]{$\mathbb{E}[(\delta\textnormal{k}_{i c})^2]$,\\$\mathbb{E}[(\delta\textnormal{v}_{i c})^2]$} & \makecell[l]{per-channel\\sensitivity} & \makecell[l]{scaling/codebook/\\SVD/MP/ TC+bit-alloc} & \makecell[l]{KIVI \cite{liu2024kivi}, KVQuant \cite{hooper2024kvquant}, KVTC\cite{staniszewski2025kv},\\PALU \cite{chang2024palu}, RateQuant \cite{zuo2026ratequant}, \textbf{Ours (Section\ref{sec:method-channel-allocation})}}  
\\ 
\bottomrule
\end{tabular}
\end{table*}

\section{AATC: Attention-Aware Bit Allocation via Transform Coding}\label{sec:method-channel-allocation}
The channel-wise allocation determines how many bits to assign to each hidden dimension of the key and value cache. As in classical transform coding, we compute it via reverse waterfilling on an empirical distribution obtained in advance from a calibration dataset. Such a set can, however, only capture the varying importance of \emph{channels}, not of individual \emph{tokens}. The allocation, therefore, operates per token, distributing bits across channels. The procedure consists of two steps: (i)~a linear transform that decorrelates and importance-orders the dimensions, and (ii)~an optimization that assigns bits to each transformed dimension so as to minimize the attention-aware distortion under a global budget.
\subsection{Step 1: Whitening-Based Feature Decorrelation}\label{subsec:whitening}
KV caches are especially amenable to transform coding, since key and value vectors are strongly correlated across dimensions and their effective rank is typically below  $d_k$~\cite{chang2024palu, staniszewski2025kv}.
We achieve decorrelation by computing a whitening-based SVD of the weight matrices $\mathbf{W}_K$ and $\mathbf{W}_V$, which approximately satisfy the independence assumptions required for classical transform coding.
Given calibration activations $\mathbf{X} \in \mathbb{R}^{T \times d_{\mathrm{model}}}$ and a projection matrix $\mathbf{W} \in \mathbb{R}^{d_{\mathrm{model}} \times d}$ (either $\mathbf{W}_K$ or $\mathbf{W}_V$), we proceed as follows to compute a transformed representation $\mathbf{H}$ (denoted by $\mathbf{H}^K$ or $\mathbf{H}^V$, respectively).   Define the empirical input covariance $\boldsymbol{\Sigma}_X = \frac{1}{T} \mathbf{X}^\top \mathbf{X}$.
Let $\boldsymbol{\Sigma}_X = \mathbf{L} \mathbf{L}^\top$ be its Cholesky decomposition.
The \emph{whitened weight matrix} is $\widetilde{\mathbf{W}} = \mathbf{L}^\top \, \mathbf{W}$,
which absorbs the input statistics into the projection. Let $\widetilde{\mathbf{W}} = \mathbf{U} \mathbf{S} \mathbf{P}^\top$ be its reduced SVD.
We then factorize the original projection as
\begin{align}\label{eq:AB-factorization}
    \mathbf{W} &= \mathbf{L}^{- \top} \mathbf{U} \mathbf{S} \mathbf{P}^\top = \mathbf{A} \, \mathbf{B},\\
    \text{with} \quad \mathbf{A} &= \mathbf{L}^{- \top} \mathbf{U} \sqrt{\mathbf{S}}, \quad \mathbf{B} = \sqrt{\mathbf{S}}\, \mathbf{P}^\top. \notag
\end{align}
At inference time, instead of caching $\mathbf{K} = \mathbf{X} \mathbf{W}_K \in \mathbb{R}^{T \times d_k}$, we cache the transformed representation
\begin{equation}\label{eq:transformed-cache}
    \mathbf{H}^K = \mathbf{X} \mathbf{A}^K \in \mathbb{R}^{T \times d_k},
\end{equation}
and recover the original keys via $\mathbf{K} = \mathbf{H}^K \mathbf{B}^K$. Same for the values, we have $\mathbf{H}^V = \mathbf{X} \mathbf{A}^V$ and therefore $\mathbf{V} = \mathbf{H}^V \mathbf{B}^V$. In the transformed space, the dimensions of $\mathbf{H}^K$ and $\mathbf{H}^V$ are decorrelated and ordered by the singular values $\mathbf{S}  = \mathbf{A}^\top \mathbf{\Sigma}_X \mathbf{A}=\mathbf{H}^T \mathbf{H}$, which capture the combined effect of input variance and projection magnitude. High-index dimensions correspond to small singular values and contribute little to the output, making them natural candidates for low-precision or zero-bit allocation.
\begin{remark}\label{rem:head-group}
In practice, one often considers a so-called multi-head attention, where each $\mathbf{W}$ is split into $H$ submatrices ($\mathbf{W}_1,...,\mathbf{W}_H \in \mathbb{R}^{d_{in}\times d_k/H}$) referred to as heads and computes the attention individually for each head before eventually combining it again. For the transform coding step in such a setup, the heads get combined into $G$ head-groups and therefore we have $\mathbf{B}_1,..., \mathbf{B}_G$. Further, $\mathbf{B}_i^K \in \mathbb{R}^{d_k/G}$, same for the values we have $\mathbf{B}_i^V \in \mathbb{R}^{d_v/G}$. $\mathbf{A}$, on the other hand, remains global.
\end{remark}
The above transform is the same as the one used in  \cite{chang2024palu} in the context of dimension reduction and its decorrelation step coincides with the PCA of $\mathbf{K}$, hence with the transform of a per-layer KVTC \cite{staniszewski2025kv}, for details see Appendix \ref{subsec:svdequalwhitening}. The equivalence concerns the transform only: KVTC operates on a global, cross-layer cache and allocates based on reconstruction error, whereas we operate per layer under a global bit budget and allocate based on the output-aware distortion of Theorem~\ref{thm:distortion}.

\subsection{Step 2: Attention-Aware Bit Allocation via Waterfilling}\label{subsubsec:bit-allocation}
With the transform in hand, we formulate the channel-wise bit allocation as a
constrained optimization problem. A detailed derivation of it can be found in Appendix \ref{sec:derivation-aatc-waterfillling-appendix}.

\paragraph{The waterfilling problem for the keys and values}
The channel-wise allocation for the keys is the solution to
\begin{align}\label{eq:waterfilling-keys}
    &\min_{\{b_{\ell, c}\}} \sum_{\ell=1}^{L}  \sum_{c=1}^{d_k} w_{\ell,c}^{(K)}  (\sigma_{\ell,c}^{(K)})^2 2^{-2b_{\ell,c}}
    \text{ s.t. } \sum_{\ell=1}^{L}  \sum_{c=1}^{d_k} b_{\ell,c} \leq \mathcal{B}_K,\\
    &\text{ with }w_{\ell,c}^{(K)} = \frac{1}{T}  \sum_{i=1}^T ( (\mathbf{B}^{K}_{\ell,c})^{\top} q_{( i,\ell,c)})^2 \notag
\end{align}
a per-channel weight with $i$ being the tokens from the calibration set, $T$ the size of the calibration set and $b_{\ell,c} \in \{0, 1, \dots, b_{\max}\}$
where $\mathcal{B}_K$ is the total bit budget for the keys.
The continuous relaxation of~\eqref{eq:waterfilling-keys} admits a weighted variant of the classical closed-form \emph{reverse waterfilling} solution~\cite[Ch.~10]{cover1999elements}. A channel receives $b_{\ell, c}^* = \tfrac{1}{2}\log_2(w_{\ell,c}^{(K)} \sigma_{\ell,c}^{(K)})^2 / \lambda)$ bits when $w_{\ell,c}^{(K)} (\sigma_{\ell,c}^{(K)})^2 > \lambda$ and zero bits otherwise, with the water level $\lambda > 0$ chosen so that $\sum_\ell \sum_c b_{\ell, c}^* = \mathcal{B}_K$. Channels below the water level are discarded entirely.
The waterfilling problem for the values reads as
\begin{align*}
    &\min_{\{b_{\ell, c}\}} \sum_{\ell=1}^{L} \sum_{c=1}^{d_v} w_{\ell,c}^{(V)}  (\sigma_{\ell,c}^{(V)})^2 2^{-2b_{\ell,c}}
    \text{ s.t. } \sum_{\ell=1}^{L} \sum_{c=1}^{d_v} b_{\ell,c} \leq \mathcal{B}_V,
\end{align*}
with  $w_{\ell,c}^{(V)} = \lVert (\mathbf{B}^V \mathbf{W}_O)_{\ell,c}^\top \rVert_F^2$ and $b_{\ell,c} \in \{0, 1, \dots, b_{\max}\}$.
\paragraph{The right scaling}
After fixing the bits per channel, one needs to think about how to choose the scale to apply uniform scalar quantization (Definition \ref{def:uniform-quant}) to the fixed number of bits. This we do in the following way: we first calculate a channel scale $(s_{\ell,c})$ on the calibration data by calculating the 99th percentile of $|h_{\ell,c}|$ across all calibration tokens for every channel existing for keys and for values. During inference time, we compute the token-wise scaling for every token in every layer as $\alpha_{\ell,t} = \max_c(|h_{\ell,c,t}/s_{\ell,c}|)$ and quantize
\begin{align*}
\mathsf{qtz}(h) &= \mathsf{clamp}\!\left(\Big\lfloor \tfrac{h}{\Delta_{\ell,c,t}} + \tfrac{1}{2} \Big\rfloor,\, -2^{b_{\ell,c}-1},\, 2^{b_{\ell,c}-1}-1\right), \\
\mathsf{dqtz}(c) &= \Delta_{\ell,c,t}\, \mathsf{qtz}(h) \text{ with } \Delta_{\ell, c,t} := \frac{2\alpha_{\ell,t}\, s_{\ell,c}}{2^{b_{\ell,c}}}.
\end{align*}
In practice we actually compute the token-wise scaling per head-group $g$ (defined in Remark \ref{rem:head-group}), so we obtain $\alpha_{\ell,g,t}=\max_{c \in g}(|h_{\ell,c,t}/s_{\ell,c}|)$
\paragraph{Separate budgets for keys and values.}
Keys and values play asymmetric roles in attention, this is why we solve separate allocation problems for $\mathbf{W}_K$ and $\mathbf{W}_V$, each with its own budget $\mathcal{B}_K$ and $\mathcal{B}_V$. 

\begin{remark}[Zero-bit dimensions and rank reduction]
The dynamic program naturally assigns zero bits to trailing dimensions whose weight $w_c$ is too small to justify even one bit, which is equivalent to rank reduction. The effective cache dimensionality is reduced from $d_k$ to the number of dimensions receiving at least one bit. Unlike fixed-rank SVD truncation like PALU~\cite{chang2024palu}, the cutoff adapts to the attention-aware weight profile and the available budget. In contrast to our method, PALU does not optimize bit allocation. In the experiments, we also see favorable performance for our method compared to PALU (see Section \ref{sec:experiments}).
\end{remark}

To summarize, our method operates in two stages: an offline calibration stage that determines the channel-wise bit allocation, and an online inference stage that applies quantization during inference.
We summarize the full procedure of our method, called attention-aware transform coding (AATC), in Algorithm~\ref{alg:aatc}.
  \begin{algorithm}[t]
  \caption{Attention-Aware Transform Coding (AATC) for KV cache Compression}\label{alg:aatc}
  \begin{algorithmic}[1]
  \Require Model weights $\{\mathbf{W}_K^{(\ell)}, \mathbf{W}_V^{(\ell)}, \mathbf{W}_O^{(\ell)}\}$; calibration data $\mathbf{X}$;
  bit budget $\mathcal{B}$; min bits $b_{\min}$
  \Ensure Compressed KV cache with attention-aware bit allocation
  \Statex \textbf{--- Calibration (offline, once per model) ---}
  \For{each layer $\ell$}
      \State Compute $\boldsymbol{\Sigma}_X^{(\ell)}$, Cholesky $\mathbf{L}^{(\ell)}$
      \State SVD of full $\mathbf{W}_K^{(\ell)}$: obtain $\mathbf{A}^{K}_{\ell}$ and
  $\mathbf{B}^{K}_{\ell}$
      \State Run forward pass; record latent variance $(\sigma_{\ell,c}^{(K)})^2$ for each $c \in \{1,\ldots,d_k\}$
          \State Compute $w_{\ell,c}^{(K)} \gets \frac{1}{T}\sum_i \bigl(\mathbf{B}^{K}_{\ell}[c,:]\cdot q_i\bigr)^2$
          \State Compute $\mathbf{B}^{V}_{\ell}$, $\sigma_{\ell,c}^{(V)}$, $ w_{\ell,c}^{(V)}  \gets \lVert (\mathbf{B}^{V}\mathbf{W}_O)_{\ell,c}^\top \rVert_F^2$
  \EndFor
  \State Allocate bits globally via reverse waterfilling over all $(\ell,c)$
  \[
      \{b_c^{(K,\ell)}\} \;\gets\; \operatorname*{arg\,min}_{\sum_{\ell,c} b_c^{(K)} \leq \mathcal{B}_K}\; \sum_{\ell,c} \Bigl(w_{\ell,c}^{(K)} \cdot (\sigma_{\ell,c}^{(K)})^2 2^{-2b_{\ell,c}}   \Bigr)
  \]
\[
      \{b_c^{(V,\ell)}\} \;\gets\; \operatorname*{arg\,min}_{\sum_{\ell,c}
  b_c^{(V)} \leq \mathcal{B}_V}\; \sum_{\ell,c} \Bigl( w_{\ell,c}^{(V)} \cdot (\sigma_{\ell,c}^{(V)})^2 2^{-2b_{\ell,c}} \Bigr)
  \]

  \Statex \textbf{--- Inference (online, per decode step $t$) ---}
  \For{each new token at step $t$}
      \For{each layer $\ell$}
          \State Project: $h^K_{\ell,t} \gets x_t^{(\ell)} \mathbf{A}^{K}_{\ell}$,\; $h^V_{\ell,t} \gets
  x_t^{(\ell)} \mathbf{A}^{V}_{\ell}$
\State Append $\mathsf{qtz}(h^K_{\ell,t})$ and $\mathsf{qtz}(h^V_{\ell,t})$ to cache
          \State Reconstruct $K_i^{(\ell)} \gets  \hat{h}^K_{\ell,i} \mathbf{B}^{K}_{\ell}$ $V_i^{(\ell)} \gets \hat{h}^V_{\ell,i}
 \mathbf{B}^{V}_{\ell} $
          \State Compute $\alpha_{t,i} \gets \mathsf{softmax}(q_t K_i^{\top} / \sqrt{d})$
          \State Output
  $o_t^{(\ell)} \gets \sum_{i \leq t} \alpha_{t,i}\, V_i^{(\ell)}$
 \EndFor
  \EndFor
  \end{algorithmic}
  \end{algorithm}

\section{Experiments}
\label{sec:experiments}
To show the performance of our method we evaluate it on \texttt{LongBench} \cite{bai2024longbench}, \texttt{RULER} \cite{hsieh2024ruler}, \texttt{GSM8K} \cite{cobbe2021training}, \texttt{MATH-500} \cite{lightman2023lets} and \texttt{MMLU-Pro} \cite{wang2024mmlu} (math and computer science subtasks) to validate its qualitative performance against existing methods (KIVI, PALU and KVQuant) on \texttt{Llama-3.1-8B-Instruct} \cite{grattafiori2024llama} and \texttt{Qwen-2.5-7B-Instruct} \cite{Yang2024Qwen25TR}. For KIVI and KVQuant, we reimplemented their method and, for fairness, added a recent-token window of size $128$, which gets flushed in blocks of $32$ tokens. We also kept the first $4$ tokens in full precision. AATC is implemented as described in Algorithm \ref{alg:aatc}. We further keep the first $s$ tokens and the most recent $w$ tokens in full precision, following standard practice~\cite{xiao2023efficient, staniszewski2025kv}. We use $s = 4$ and $w = 128$. Once the recent-token window is full, we quantize the oldest $16$ tokens and refill until the window again holds $w$ tokens.  For PALU, we applied a transform with $30\%$ dimension reduction and afterwards quantized to $4$ bits. Further, we used $s=4$ and $w=128$ for the non-zero dimensions. 
We further evaluate AATC var-only, which uses the same per-layer transform and reverse-waterfilling allocation but drops the attention-aware weights  $q_c^2$ and $W_{O_c}$, allocating on the quantization-error variance alone. 
KVTC\cite{staniszewski2025kv} is the closest method to ours: it likewise decorrelates the cache and allocates bits by transform coding. It differs on two axes: it uses a single global, cross-layer transform (we use per-layer transforms under a global budget, the natural granularity given that each layer maintains its own KV cache) and it allocates on reconstruction error (we allocate on the output-aware distortion of Theorem \ref{thm:distortion}, via $q_c^2$ and $W_{O_c}$). In the per-layer setting, the two coincide: the transforms match (Appendix \ref{subsec:svdequalwhitening}) and the var-only's variance-based allocation is exactly the reconstruction error criterion, so the var-only variant above can be interpreted as a per-layer analog of the KVTC algorithm. 
For more information on the methods, see Appendix  \ref{subsec:compare-methods}.  Llama has 8 KV-heads, 32 Query heads and 32 layers. Qwen has 4 KV-heads, 28 Query heads and 28 layers. Therefore, both methods use grouped query attention. Our method we use with $G=2$ head-groups for Llama and $G=1$ head-group for Qwen.
To generate the transform $AB$ for the whitening-based feature decorrelation, we use wikitext2 \cite{merity2016pointer}. For the rate-distortion calibration, we use a balanced mixture of FineWeb \cite{penedo2024fineweb} (sample-10BT, web text) and OpenR1-Math-220k \cite{openr1} (mathematical reasoning traces). We draw 256 calibration sequences of length 2048 tokens, split equally between FineWeb and OpenR1-Math-220k (128 sequences each), and shuffle them together before calibration. The mixture is intended to provide coverage over both general-domain and  structured reasoning text distributions, following the approach of KVTC \cite{staniszewski2025kv}.

\subsection{Main Results}
\label{subsec:main-results}
 
Table \ref{tab:mainresults} reports the main comparison; we mark in bold the best compressed method. 
The FP16 model is a reference and is not bolded. For each benchmark we take $\sigma$ to be the standard error of its score, which is the standard deviation of the per-example scores divided by $\sqrt{N}$ over the $N$ evaluated samples. We call a method statistically indistinguishable from FP16 when its score lies within 2$\sigma$ of the FP16 score on that benchmark.
 
Across all $18$ cells, AATC is never statistically distinguishable from the
uncompressed FP16 model. At $5.82\times$ compression it is lossless to within
evaluation noise on every benchmark and both backbones. This operating points reduces the memory budget of Llama from 1.07 GB for the FP16 baseline to 184MB. On Qwen, it is moreover the
unique best compressed method on long-context retrieval (\texttt{RULER}-32k) and on
multiple-choice reasoning (\texttt{MMLU-Pro} math), exceeding the strongest quantization
baseline by $7.2$ and $16.8$ points. The two backbones respond differently to
compression. On Llama, where even naive quantization (KIVI) stays near-lossless,
AATC and AATC var-only are statistically tied throughout. On Qwen, which is
markedly harder to compress (KIVI falls to $0.351$ on \texttt{LongBench} and below $0.28$
on \texttt{RULER}), the additional adaptivity of the full method gives consistent gains
over AATC var-only (\texttt{RULER}-32k $+2.5$, \texttt{MMLU}-m $+3.6$). PALU is competitive only at short contexts, where it is best at \texttt{RULER}-4k on Qwen, but it attains the weakest
\texttt{LongBench} average ($0.483$), a behavior we attribute to its low-rank
parameterization. For a detailed explanation, see Appendix \ref{subsec:appendix-longbench}.
 
We report further analyses in the appendix: a per-subtask \texttt{LongBench}
breakdown (in Appendix \ref{subsec:appendix-longbench}), an average bit sweep from 4-bit to 1-bit
(Appendix \ref{subsec:appendix-bitsweep}), detailed results at the
aggressive 2-bit setting (Appendix \ref{subsec:appendix-bitsweep}), an ablation of the allocation and
calibration choices (Appendix \ref{subsec:appendix-ablation}), full
implementation details for all methods (Appendix \ref{subsec:compare-methods}) and details about the benchmarks (Appendix \ref{subsec:benchmarks-quality}).

\subsection{Benefits of Attention-Awareness}
\label{subsec:divergence}
 
At the $5.82\times$ operating point of Table \ref{tab:mainresults}, AATC var-only is
statistically indistinguishable from full AATC for Llama, which raises the question of
whether including the $q_c^2$ and $W_O$ terms in the bit-allocation is necessary for all models. Figure \ref{fig:divergence} answers this by
evaluating both variants at a more aggressive $\approx\!7\times$ compression, corresponding to a 2-bit budget.
At 2.5-bit (solid), the two curves are indistinguishable at every context length for Llama. For Qwen we already see a small advantage of AATC.
At 2-bit (dashed) the full method retains near-baseline \texttt{RULER} accuracy across
lengths for both models, whereas AATC var-only degrades sharply as the context grows. The more we compress, the better we can see the effect of the attention-aware allocation. We see, especially under aggressive compression and in long-context regimes, which is what we care about.

\begin{figure}
    \centering
    \includegraphics[width=1.0\linewidth]{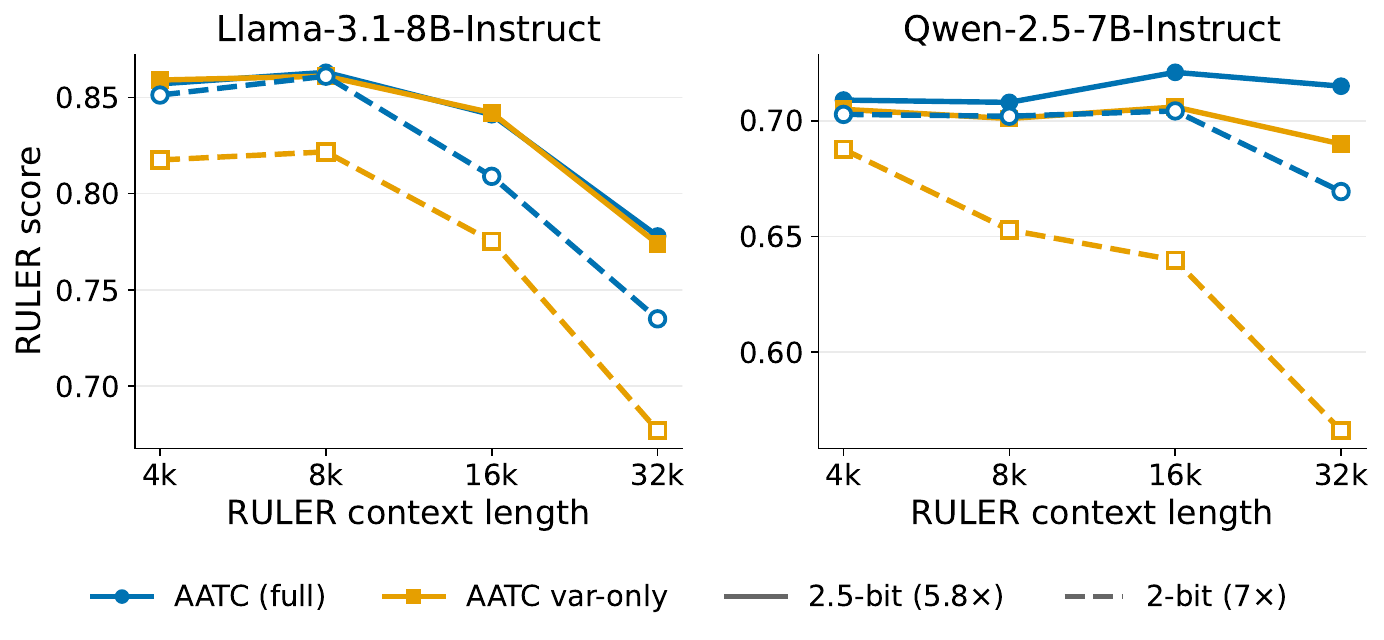}
    \caption{\texttt{RULER} score versus context length for AATC and AATC var-only at
2.5-bit (solid) and 2-bit (dashed), on \texttt{Llama-3.1-8B-Instruct} and \texttt{Qwen-2.5-7B-Instruct}. }
    \label{fig:divergence}
\end{figure}

\begin{table*}[ht]
\centering
\small
\setlength{\tabcolsep}{4pt}
\renewcommand{\arraystretch}{1.0}
\caption{Comp.\ is the KV cache compression ratio relative to the FP16 baseline. \texttt{LongBench} is the task-averaged score over seven English subsets (per-subtask breakdown in Table~\ref{tab:longbench-full}). Here $\sigma$ is the standard error of each score, the standard deviation of the per-example scores divided by $\sqrt{N}$ over the $N$ evaluated samples. A method is called statistically indistinguishable from FP16 when its score lies within $2\sigma$ of the FP16 score.} 
\label{tab:mainresults}
\begin{tabular}{l | c | c | cccc | cccc}
\toprule
& & \textbf{\texttt{LongBench}} & \multicolumn{4}{c|}{\textbf{\texttt{RULER}}} & \multicolumn{4}{c}{\textbf{Reasoning}} \\
Method & Comp. & avg. & 4k & 8k & 16k & 32k & \texttt{GSM8K} & \texttt{MMLU}-m & \texttt{MMLU}-cs & \texttt{MATH-500} \\
\midrule
\multicolumn{11}{c}{\texttt{Llama-3.1-8B-Instruct}} \\
\midrule
Baseline    & 1.00 & 0.570 & 0.859 & 0.856 & 0.828 & 0.775 & 0.789 & 0.438 & 0.456 & 0.454 \\
KIVI        & 5.00 & 0.564 & 0.846 & 0.840 & 0.810 & 0.763 & 0.730 & 0.390 & 0.420 & 0.452 \\
PALU        & 5.39 &  0.550 & 0.851   & 0.849     &  0.83     &     0.769   &   0.617    &  0.274     &  0.320     &     0.314  \\
KVQuant     & 5.91 & 0.549 & 0.715 & 0.729 & 0.676 & 0.654 & 0.754 & 0.376 & 0.402 & 0.392 \\
AATC var    & 5.82 & \textbf{0.566} & \textbf{0.859} & 0.861 & \textbf{0.842} & 0.774 & \textbf{0.780} & 0.402 & \textbf{0.449} & \textbf{0.490} \\
AATC (ours) & 5.82 & 0.558 & 0.857 & \textbf{0.863} & 0.841 & \textbf{0.778} & 0.773 & \textbf{0.416} & 0.442 & 0.488 \\
\midrule
\multicolumn{11}{c}{\texttt{Qwen-2.5-7B-Instruct}} \\
\midrule
Baseline    & 1.00 & 0.574 & 0.720 & 0.716 & 0.718 & 0.720 & 0.818 & 0.708 & 0.567 & 0.680 \\
KIVI        & 5.00 & 0.351 & 0.279 & 0.224 & 0.277 & 0.233 & 0.165 &        0.166 &         0.222      & \textbf{0.680} \\
PALU        & 5.39 & 0.483 & \textbf{0.725} & 0.690 & 0.682 &   0.625     & 0.672 &  0.542   &       0.454   & 0.492 \\
KVQuant     & 5.91 & 0.558 & 0.694 & 0.674 & 0.667 & 0.643 & 0.801 & 0.478 & 0.490 & \textbf{0.680}    \\
AATC var    & 5.82 & 0.558 & 0.705 & 0.701 & 0.706 & 0.690 & \textbf{0.824} & 0.674 & 0.551 & 0.668 \\
AATC (ours) & 5.82 & \textbf{0.567} & 0.709 & \textbf{0.708} & \textbf{0.721} & \textbf{0.715} & 0.822 & \textbf{0.710} & \textbf{0.573} & \textbf{0.680} \\
\bottomrule
\end{tabular}
\end{table*}

\section{Limitations and Future Work}
On the theoretical side, our analysis assumes that tokens are independent, which may not always hold in practice, since natural language can exhibit non-trivial correlation structures. While this assumption is common in related works and does not seem to have a strong negative impact on the encoding quality, we still think characterizing the dependence, or deriving guarantees under weaker assumptions, is an interesting direction for future theoretical work.
On the empirical side, our implementation is not system-level optimized, in
particular, it lacks a dedicated CUDA implementation, which would be required
before the method could be used in production. Our evaluation also covers only
two model families (Llama and Qwen), while these differ in scale and attention configuration, testing on a wider range of models is left for future work.
Finally, it would be interesting to study how our method interacts with
token-eviction methods~\cite{zhang2023h2o, xiao2023efficient, li2024snapkv}
or how it can be extended to a global cache as it is deployed in in~\cite{staniszewski2025kv}.
\newpage

\appendix
\subsection{Proof of Theorem \ref{thm:distortion}}
\label{sec:softmax-derivatives}

\begin{proof}[Proof of Theorem~\ref{thm:distortion}]
\label{sec:proofofdistortionformula}
We proceed in four steps: (1)~express the output error in terms of the perturbations
$\delta\textnormal{a}_i$, $\delta\textnormal{v}_i$; (2)~show that the cross-terms vanish or are of higher order in
expectation; (3)~compute the key distortion term; (4)~compute the value distortion term.
Throughout, $C$ denotes an absolute constant depending only on $q$, $\mathbf{W}_O$, and the attention
weights $\{a_i\}$. Its value may change from line to line. We write
$\sigma_K^2 := \max_{i,c}\mathbb{E}[(\delta\textnormal{k}_{i c})^2]$ and
$\sigma_V^2 := \max_{i,c}\mathbb{E}[(\delta\textnormal{v}_{i c})^2]$.

\medskip
\noindent\textbf{Step 1: Decomposition of the output error.}
Let $\hat{\textnormal{a}}_i$ denote the attention weights computed from the
quantized keys and define the attention-weight perturbation
$\delta\textnormal{a}_i := \hat{\textnormal{a}}_i - a_i$,
so that $\hat{\textnormal{a}}_i = a_i + \delta\textnormal{a}_i$ .
The unperturbed and perturbed outputs are
$o = \sum_{i=1}^{T} a_i\, v_i\, \mathbf{W}_O$ and $ \hat{\textnormal{o}} = \sum_{i=1}^{T} \hat{\textnormal{a}}_i\, \hat{\textnormal{v}}_i\, \mathbf{W}_O
= \sum_{i=1}^{T} (a_i + \delta\textnormal{a}_i)(v_i + \delta\textnormal{v}_i)\, \mathbf{W}_O$ .
Multiplying out the product $(a_i + \delta\textnormal{a}_i)(v_i + \delta\textnormal{v}_i) = a_i v_i + a_i\,\delta\textnormal{v}_i
+ \delta\textnormal{a}_i\, v_i + \delta\textnormal{a}_i\,\delta\textnormal{v}_i$ and subtracting, the unperturbed term $a_i v_i$
cancels against the corresponding term in $o$:
\begin{align*}
o - \hat{\textnormal{o}}
=& \sum_{i=1}^{T} a_i v_i \mathbf{W}_O\\ &- \sum_{i=1}^{T} \bigl[a_i v_i + a_i \delta\textnormal{v}_i
   + \delta\textnormal{a}_i v_i + \delta\textnormal{a}_i \delta\textnormal{v}_i\bigr] \mathbf{W}_O \nonumber \\
= &-\Bigl( \underbrace{\sum_{i=1}^{T} \delta\textnormal{a}_i v_i \mathbf{W}_O}_{=:\, E_K}
+ \underbrace{\sum_{i=1}^{T} a_i \delta\textnormal{v}_i \mathbf{W}_O}_{=:\, E_V}
+ \underbrace{\sum_{i=1}^{T} \delta\textnormal{a}_i \delta\textnormal{v}_i \mathbf{W}_O}_{=:\, E_{\times}} \Bigr),
\end{align*}
and consequently
\begin{equation}\label{eq:norm-equals}
\mathbb{E}\bigl[\|o - \hat{\textnormal{o}}\|^2\bigr] = \mathbb{E}\bigl[\|E_K + E_V + E_\times\|^2\bigr].
\end{equation}
The term $E_K$ captures the effect of key quantization (through the perturbed attention
weights), $E_V$ captures the effect of value quantization and
$E_\times$ is a second-order cross-term involving both perturbations.

\medskip
\noindent\textbf{Softmax expansion.}
Both the key distortion (computed in Step 3 below) and the cross-terms (computed in Step 2 below) require the response of the
attention weights to the key perturbation. Because the attention logits are given by $q^\top k_i$, the quantized keys perturb them by
\begin{align*}
q^\top \hat{\textnormal{k}}_i = q^\top(k_i + \delta\textnormal{k}_i) = q^\top k_i + \alpha_i,
\\ \text{where }
\alpha_i := q^\top \delta\textnormal{k}_i = \sum_{c} q_c \delta\textnormal{k}_{i c},
\end{align*}
so $\alpha_i$ is the scalar perturbation to the $i$-th logit induced by the key error. Under Assumption~\ref{ass:white-noise} the $\delta\textnormal{k}_{i c}$ are
independent and zero-mean, hence so are the $\alpha_i$, so, in particular, $\mathbb{E}[\alpha_i \alpha_\ell] = 0$
for $i \neq \ell$. Furthermore, expanding the square of the sum and using channel independence,
 $\mathbb{E}[\delta\textnormal{k}_{i c}\delta\textnormal{k}_{i c'}] = \mathbbm{1}_{c=c'}\mathbb{E}[(\delta\textnormal{k}_{i c})^2]$, one obtains that,
\begin{equation}\label{eq:alpha-sq}
\mathbb{E}[\alpha_i^2]
= \mathbb{E}\Bigl[\bigl(\textstyle\sum_c q_c\,\delta\textnormal{k}_{i c}\bigr)^2\Bigr]
= \sum_c q_c^2\, \mathbb{E}[(\delta\textnormal{k}_{i c})^2]
\;\le\; \|q\|^2 \sigma_K^2.
\end{equation}
The last step uses that $\mathbb{E}[(\delta\textnormal{k}_{i c})^2] \le \sigma_K^2$ and
$\sum_c q_c^2 = \|q\|^2$. Moreover, as a linear combination of independent symmetric variables $\alpha_i$ is itself
symmetric and all of its odd moments vanish, in particular
$\mathbb{E}[\alpha_i] = \mathbb{E}[\alpha_i^3] = 0$.
Under Assumption~\ref{ass:bounded}, the logit perturbations are moreover bounded, which means $|\alpha_i| \le \sum_c |q_c|\,|\delta\textnormal{k}_{i c}| \le  \kappa\|q\|_1\sigma_K := A_0$ almost surely.
The perturbations will be centered at
the attention-weighted mean
$\bar\alpha := \sum_{\ell=1}^{T} a_\ell\,\alpha_\ell.$
The weights $a_\ell$ being the unperturbed attention weights leads to $a_\ell \ge 0$ and $\sum_\ell a_\ell = 1$.
Since $\bar\alpha$ is a convex combination of the $\alpha_\ell$, also $|\bar\alpha| \le A_0$ and
$|\alpha_i - \bar\alpha| \le 2A_0$ almost surely.
Writing $s_\ell:= q^\top k_\ell$ for the unperturbed logits, the attention weight is
$\hat{\textnormal{a}}_i = \mathsf{softmax}(s + \alpha)_i$. Decompose the attention-weight perturbation as
\begin{align}\label{eq:softmax-2nd}
\delta\textnormal{a}_i = &\underbrace{a_i\bigl(\alpha_i - \bar\alpha\bigr)}_{=:\,\delta\textnormal{a}_i^{(1)}}
+ \underbrace{\tfrac{1}{2} a_i\Bigl[\bigl(\alpha_i - \bar\alpha\bigr)^2
- \sum_{\ell} a_\ell\bigl(\alpha_\ell - \bar\alpha\bigr)^2\Bigr]}_{=:\,\delta\textnormal{a}_i^{(2)}} +\textnormal{r}_i
\end{align}
where $\delta\textnormal{a}_i^{(1)}$ is linear, $\delta\textnormal{a}_i^{(2)}$ quadratic in the entries of $\alpha$ and
$\textnormal{r}_i := \delta\textnormal{a}_i - \delta\textnormal{a}_i^{(1)} - \delta\textnormal{a}_i^{(2)}$ is defined as the exact difference. By
Remark~\ref{rem:softmax-expansion}, $\delta\textnormal{a}_i^{(1)}$ and $\delta\textnormal{a}_i^{(2)}$ are precisely the
first- and second-order Taylor terms of the softmax response
$\theta \mapsto \mathsf{softmax}(s+\theta\alpha)_i$ at $\theta = 0$, so $\textnormal{r}_i$ is its exact
second-order Taylor remainder. It is bounded in Lemma~\ref{lem:softmax-tail} below.
Because the $\hat{\textnormal{a}}_i$ and $a_i$ are
both probability distributions ($\sum_i \hat{\textnormal{a}}_i = \sum_i a_i = 1$), their difference satisfies
\begin{equation}\label{eq:zero-sum}
\sum_i \delta\textnormal{a}_i = \sum_i(\hat{\textnormal{a}}_i - a_i) = 0 .
\end{equation}
We use this expansion throughout Steps 2--4.

\medskip
\noindent\textbf{Step 2: Controlling the cross-terms in expectation.}
Expanding the squared norm in~\eqref{eq:norm-equals} yields,
\begin{align}\label{eq:squared-expansion}
\|E_K + E_V + E_\times\|^2 = &\|E_K\|^2 + \|E_V\|^2 + \|E_\times\|^2\\
&+ 2 E_K E_V^\top  + 2 E_K E_\times^\top  + 2 E_V E_\times^\top \notag .
\end{align}
We will now show that $\mathbb{E}[E_K E_V^\top ]$ and $\mathbb{E}[E_K E_\times^\top ]$ vanish exactly, while
$\mathbb{E}[E_V E_\times^\top ]$ and $\mathbb{E}[\|E_\times\|^2]$ are of higher order and collected
into the remainder. The leading terms $\mathbb{E}[\|E_K\|^2]$ and
$\mathbb{E}[\|E_V\|^2]$ are evaluated in Steps 3 and 4.
\smallskip

\noindent\emph{(a) The term $\mathbb{E}[E_K E_V^\top ]$ vanishes exactly.}
Writing both factors as sums gives,
\begin{align*}
\mathbb{E}[E_K E_V^\top]
&= \sum_{i=1}^{T} \mathbb{E}\left[ \sum_{\ell=1}^{T} \delta\textnormal{a}_i  a_\ell  (v_i \mathbf{W}_O) (\delta\textnormal{v}_\ell \mathbf{W}_O)^\top \right]\\
& = \sum_{i=1}^{T} \sum_{\ell=1}^{T} (v_i\, \mathbf{W}_O) \,\mathbb{E}_{\delta\textnormal{k}}[\delta\textnormal{a}_i] a_\ell \mathbf{W}_O^\top\, \mathbb{E}_{\delta\textnormal{v}}[\delta\textnormal{v}_\ell^\top]
= 0.
\end{align*}
As the perturbation $\delta\textnormal{a}_i$ does not depend on the value errors $\delta\textnormal{v}_\ell$, the independence of key and value errors (Assumption~\ref{ass:white-noise}) yields that $\delta\textnormal{a}_i$ and $\delta\textnormal{v}_\ell$ are also independent and the last row follows,
since $\mathbb{E}[\delta\textnormal{v}_\ell] = 0$.

\smallskip
\noindent\emph{(b) The term $\mathbb{E}[E_K E_\times^\top]$ vanishes exactly.}
Analogously to (a), $\delta\textnormal{a}_i\,\delta\textnormal{a}_\ell$ and $\delta\textnormal{v}_\ell$ are independent. Thus,
\begin{align*}
\mathbb{E}[E_K E_\times^\top]
&= \sum_{i=1}^{T} \sum_{\ell=1}^{T} \mathbb{E}[\delta\textnormal{a}_i\, \delta\textnormal{a}_\ell\;(v_i \mathbf{W}_O) (\delta\textnormal{v}_\ell\, \mathbf{W}_O)^\top]\\
&= \sum_{i=1}^{T} \sum_{\ell=1}^{T} \mathbb{E}_{\delta\textnormal{k}}[\delta\textnormal{a}_i\, \delta\textnormal{a}_\ell]\;(v_i \mathbf{W}_O) \mathbf{W}_O^\top\,\mathbb{E}_{\delta\textnormal{v}}[\delta\textnormal{v}_\ell^\top]
= 0.
\end{align*}

\smallskip
\noindent\emph{(c) $\mathbb{E}[E_V E_\times^\top]$ is
$\mathcal{O}\bigl((\|q\|^2\sigma_K^2 + A_0^3)\,\|\mathbf{W}_O\|_F^2\,\sigma_V^2\bigr)$.}
Expanding both factors,
\begin{align*}
\mathbb{E}[E_V^\top E_\times]
&= \sum_{i=1}^{T} \sum_{\ell=1}^{T} \mathbb{E}[a_i\, \delta\textnormal{a}_\ell\; (\delta\textnormal{v}_i\, \mathbf{W}_O) (\delta\textnormal{v}_\ell\, \mathbf{W}_O)^\top] \\
& = \sum_{i=1}^{T} \sum_{\ell=1}^{T} a_i\;\mathbb{E}_{\delta\textnormal{k}}[\delta\textnormal{a}_\ell]\;\mathbb{E}_{\delta\textnormal{v}}\bigl[(\delta\textnormal{v}_i\, \mathbf{W}_O)(\delta\textnormal{v}_\ell\, \mathbf{W}_O)^\top \bigr].
\end{align*}
The second equality holds, since $\delta\textnormal{a}_\ell$ depends only on $\{\delta\textnormal{k}_m\}$, which are independent of the value errors.
By token independence and $\mathbb{E}[\delta\textnormal{v}_i] = 0$, the value expectation vanishes for
$i \neq \ell$, so only the diagonal $i = \ell$ survives therefore we get
\[
\mathbb{E}\bigl[E_V E_\times^\top \bigr]
= \sum_{i=1}^{T} a_i\;\mathbb{E}_{\delta\textnormal{k}}[\delta\textnormal{a}_i]\;\mathbb{E}_{\delta\textnormal{v}}\bigl[\|\delta\textnormal{v}_i\, \mathbf{W}_O\|^2\bigr].
\]
It remains to evaluate $\mathbb{E}[\delta\textnormal{a}_i]$. From the expansion~\eqref{eq:softmax-2nd},
$\mathbb{E}[\delta\textnormal{a}_i^{(1)}] = a_i\,\mathbb{E}[\alpha_i - \bar\alpha] = 0$ because the $\alpha_i$ are
zero-mean (Assumption~\ref{ass:white-noise}). The leading nonzero contribution is the second-order
term, with the tail controlled by Lemma~\ref{lem:softmax-tail}:
\begin{align*}
\mathbb{E}[\delta\textnormal{a}_i] &= \mathbb{E}[\delta\textnormal{a}_i^{(2)}] + \mathbb{E}[\textnormal{r}_i],\\
\mathbb{E}[\delta\textnormal{a}_i^{(2)}]
&= \tfrac{1}{2}a_i\Bigl(\mathbb{E}[(\alpha_i - \bar\alpha)^2] - \textstyle\sum_\ell a_\ell\,\mathbb{E}[(\alpha_\ell - \bar\alpha)^2]\Bigr),
\end{align*}
where $|\mathbb{E}[\textnormal{r}_i]| \le \mathbb{E}|\textnormal{r}_i| \le 7 A_0^3$
by~\eqref{eq:tail-L1}. The centered second moments are computed exactly once here and reused in
(d) below. Using $\mathbb{E}[\alpha_\ell\bar\alpha] = a_\ell\,\mathbb{E}[\alpha_\ell^2]$ since, only the $m=\ell$
term of $\bar\alpha$ contributes, by token independence and zero mean, and
$\mathbb{E}[\bar\alpha^2] = \sum_m a_m^2\,\mathbb{E}[\alpha_m^2]$, we obtain
\begin{equation}\label{eq:centered-var}
\mathbb{E}\bigl[(\alpha_\ell - \bar\alpha)^2\bigr]
= (1-2a_\ell)\,\mathbb{E}[\alpha_\ell^2] + \sum_m a_m^2\,\mathbb{E}[\alpha_m^2]
\;\le\; 2\,\|q\|^2\sigma_K^2.
\end{equation}
Here, the final bound uses $\mathbb{E}[\alpha_m^2] \le \|q\|^2\sigma_K^2$
from~\eqref{eq:alpha-sq}, $1-2a_\ell \le 1$, and $\sum_m a_m^2 \le 1$. Hence
$|\mathbb{E}[\delta\textnormal{a}_i^{(2)}]| \le C\,a_i\,\|q\|^2\sigma_K^2$ and, in
particular, $|\mathbb{E}[\delta\textnormal{a}_i]| \le C\bigl(a_i \|q\|^2\sigma_K^2 + A_0^3\bigr)$.
Combining the above bound with
\begin{equation}\label{eq:value-factor-bound-step2c}
\mathbb{E}\bigl[\|\delta\textnormal{v}_i \mathbf{W}_O\|^2\bigr]
= \sum_c \mathbb{E}[(\delta\textnormal{v}_{i c})^2]\|(\mathbf{W}_O)_{c,:}\|^2 \le \sigma_V^2\|\mathbf{W}_O\|_F^2
\end{equation}
and the fact that the per-token
bound carries its own factor $a_i$, so that the token sum collapses through
$\sum_i a_i^2 \le \sum_i a_i = 1$, we obtain the final bound
\begin{align}\label{eq:EV-Ecross-bound}
\bigl|\mathbb{E}[E_V E_\times^\top]\bigr|
&\le \sum_i a_i\,\bigl|\mathbb{E}[\delta\textnormal{a}_i]\bigr|\,\sigma_V^2\|\mathbf{W}_O\|_F^2\\
&\le C\,\bigl(\|q\|^2\sigma_K^2 + A_0^3\bigr)\,\|\mathbf{W}_O\|_F^2\,\sigma_V^2 . \notag
\end{align}

\smallskip
\noindent\emph{(d) The term $\mathbb{E}[\|E_\times\|^2]$ is
$\mathcal{O}(\|q\|^2\|\mathbf{W}_O\|_F^2\,\sigma_K^2\sigma_V^2)$.}
Expanding the squared norm and using the independence of key and value noise,
\[
\mathbb{E}\bigl[\|E_\times\|^2\bigr]
= \sum_{i=1}^{T} \sum_{\ell=1}^{T} \mathbb{E}_{\delta\textnormal{k}}[\delta\textnormal{a}_i\, \delta\textnormal{a}_\ell]\;\mathbb{E}_{\delta\textnormal{v}}\bigl[(\delta\textnormal{v}_i\, \mathbf{W}_O) (\delta\textnormal{v}_\ell\, \mathbf{W}_O)^\top \bigr].
\]
For $i \neq \ell$, the value factor and the key factor both vanish by token independence and zero mean, leaving the
diagonal
\begin{equation}\label{eq:Ecross-squared}
\mathbb{E}\bigl[\|E_\times\|^2\bigr]
= \sum_{i=1}^{T} \mathbb{E}\bigl[\delta\textnormal{a}_i^2\bigr]\;\mathbb{E}\bigl[\|\delta\textnormal{v}_i\, \mathbf{W}_O\|^2\bigr].
\end{equation}
The value factor bound is already bounded in \eqref{eq:value-factor-bound-step2c}. Therefore, we only need to bound the key factor. Its leading part
$\delta\textnormal{a}_i^{(1)} = a_i(\alpha_i - \bar\alpha)$ given the bound in~\eqref{eq:centered-var} yields
\begin{equation*}
\mathbb{E}\bigl[(\delta\textnormal{a}_i^{(1)})^2\bigr]
= a_i^2\,\mathbb{E}\bigl[(\alpha_i - \bar\alpha)^2\bigr]
\;\le\; 2\,a_i^2\,\|q\|^2\,\sigma_K^2.
\end{equation*}
Squaring the full
$\delta\textnormal{a}_i = \delta\textnormal{a}_i^{(1)} + \delta\textnormal{a}_i^{(2)} + \textnormal{r}_i$ gives
\[
\mathbb{E}[\delta\textnormal{a}_i^2]
= \mathbb{E}[(\delta\textnormal{a}_i^{(1)})^2]
+ \underbrace{2\,\mathbb{E}[\delta\textnormal{a}_i^{(1)}(\delta\textnormal{a}_i^{(2)} + \textnormal{r}_i)] + \mathbb{E}[(\delta\textnormal{a}_i^{(2)} + \textnormal{r}_i)^2]}_{=: c_i}.
\]
The corrections $c_i$ are controlled pathwise. We have that $|\alpha_\ell - \bar\alpha| \le 2A_0$ and
$\sum_\ell a_\ell(\alpha_\ell - \bar\alpha)^2 \le 4A_0^2$ almost surely, the latter because it is a
convex combination of the terms $(\alpha_\ell - \bar\alpha)^2 \le 4A_0^2$. From those inequalities it follows directly,
from~\eqref{eq:softmax-2nd},
$|\delta\textnormal{a}_i^{(1)}| \le 2a_i A_0$ and $|\delta\textnormal{a}_i^{(2)}| \le 2a_i A_0^2$, while
$|\textnormal{r}_i| \le 7A_0^3$ and $\sum_i \mathbb{E}[\textnormal{r}_i^2] \le 49A_0^6$ by
Lemma~\ref{lem:softmax-tail}. Summing over tokens with $\sum_i a_i = 1$ and
$\sum_i a_i^2 \le 1$,
\[
\sum_i |c_i| \;\le\; 8A_0^3 + 28A_0^4 + 8A_0^4 + 98A_0^6 \;\le\; C\,A_0^3 ,
\]
the final inequality holding for $A_0 \le 1$ directly, and for $A_0 > 1$ because
$\sum_i \mathbb{E}[\delta\textnormal{a}_i^2] \le \mathbb{E}[(\sum_i|\delta\textnormal{a}_i|)^2] \le 4$ and
$\sum_i \mathbb{E}[(\delta\textnormal{a}_i^{(1)})^2] \le 2A_0^2$ give
$\sum_i |c_i| \le 4 + 2A_0^2 \le 6A_0^3$ outright.
Therefore, the bound on the attention distortion terms is
$\sum_i \mathbb{E}[\delta\textnormal{a}_i^2] \le C\bigl(\|q\|^2\sigma_K^2 +A_0^3\bigr)$.
Substituting this and the value factor into~\eqref{eq:Ecross-squared} and using
$\sum_i a_i^2 \le 1$,
\begin{equation}\label{eq:Ecross-final-bound}
\mathbb{E}\bigl[\|E_\times\|^2\bigr]
\;\le\; C\,\bigl(\|q\|^2\sigma_K^2 + A_0^3\bigr)\,\|\mathbf{W}_O\|_F^2\;\sigma_V^2 .
\end{equation}
\smallskip
Summarizing Step~2: in expectation the expansion~\eqref{eq:squared-expansion} retains only the
diagonal terms $\mathbb{E}[\|E_K\|^2]$ and $\mathbb{E}[\|E_V\|^2]$ together with a controlled
remainder. The cross-terms (a) and (b) vanish exactly at all orders, while (c) and (d) are
bounded by~\eqref{eq:EV-Ecross-bound} and~\eqref{eq:Ecross-final-bound}. Hence
\begin{align*}
\mathbb{E}\bigl[\|o - \hat{\textnormal{o}}\|^2\bigr]
&= \mathbb{E}\bigl[\|E_K\|^2\bigr] + \mathbb{E}\bigl[\|E_V\|^2\bigr] + R_2,\\
R_2 &:= \mathbb{E}\bigl[\|E_\times\|^2\bigr] + 2\,\mathbb{E}\bigl[E_V E_\times^\top\bigr] \notag\\
&= \mathcal{O}\Bigl(\bigl(\|q\|^2\sigma_K^2 + A_0^3\bigr)\|\mathbf{W}_O\|_F^2\,\sigma_V^2\Bigr).\notag
\end{align*}
\medskip
\noindent\textbf{Step 3: The key distortion term $\mathbb{E}[\|E_K\|^2]$.}
We insert the softmax expansion~\eqref{eq:softmax-2nd} into $E_K = \sum_i \delta\textnormal{a}_i\, v_i \mathbf{W}_O$.
By~\eqref{eq:zero-sum}, $\sum_i \delta\textnormal{a}_i = 0$, so for the constant vector $o$ we have
$\sum_i \delta\textnormal{a}_i\, o = \bigl(\sum_i \delta\textnormal{a}_i\bigr)o = 0$, subtracting this from $E_K$ lets us
recenter each $v_i \mathbf{W}_O$ by $o$ without changing the value.
Write $u_i := v_i \mathbf{W}_O - o$ and $D_o := \max_i \|u_i\|^2$. Grouping the softmax expansion as
$\delta\textnormal{a}_i = \delta\textnormal{a}_i^{(1)} + \delta\textnormal{a}_i^{(2)} + \textnormal{r}_i$
Lemma~\ref{lem:softmax-tail},
\begin{equation*}
E_K = \sum_{i=1}^{T} \delta\textnormal{a}_i\,u_i
= \sum_{i=1}^{T} \bigl(\delta\textnormal{a}_i^{(1)} + \delta\textnormal{a}_i^{(2)}\bigr)u_i
+ \underbrace{\sum_{i=1}^{T} \textnormal{r}_i\, u_i}_{=:\,E_K^{(3+)}} .
\end{equation*}
Writing $E_K^{(m)} := \sum_i \delta\textnormal{a}_i^{(m)} u_i$, so that  $E_K = E_K^{(1)} + E_K^{(2)} + E_K^{(3+)}$ exactly with 
 $E_K^{(3+)} := \sum_i \textnormal{r}_i u_i$ . We first evaluate the leading term exactly. Substituting
$\delta\textnormal{a}_i^{(1)} = a_i(\alpha_i - \bar\alpha)$ and using $\sum_i a_i u_i = \sum_i a_i (v_i \mathbf{W}_O - o)= \sum_i a_i v_i \mathbf{W}_O - o = o - o = 0$ (the middle step holds by $\sum_i a_i=1$), the mean-subtraction term drops out:
\begin{align*}
E_K^{(1)} &= \sum_i a_i(\alpha_i - \bar\alpha)\,u_i
= \sum_i a_i \alpha_i\, u_i - \bar\alpha\underbrace{\sum_i a_i u_i}_{=\,0}\\
&= \sum_i a_i \alpha_i\, u_i .
\end{align*}
Therefore, using $\mathbb{E}[\alpha_i\alpha_\ell] = \delta_{i\ell}\mathbb{E}[\alpha_i^2]$,
\begin{align}\label{eq:EK-final}
\mathbb{E}\bigl[\|E_K^{(1)}\|^2\bigr]
&= \sum_{i,\ell} a_i a_\ell\, \mathbb{E}[\alpha_i \alpha_\ell]\, u_i u_\ell^\top
= \sum_{i} a_i^2\, \mathbb{E}[\alpha_i^2]\, \|u_i\|^2 \notag\\
&= \sum_{i,c} a_i^2\, \|v_i \mathbf{W}_O - o\|^2\, q_c^2\, \mathbb{E}[(\delta\textnormal{k}_{i c})^2],
\end{align}
where the last step inserts $\mathbb{E}[\alpha_i^2] = \sum_c q_c^2\,\mathbb{E}[(\delta\textnormal{k}_{i c})^2]$
from~\eqref{eq:alpha-sq}. This is precisely $D_K$.
In particular, by $\mathbb{E}[\alpha_i^2]\leq \lVert q \rVert^2 \sigma_K^2$ and $\sum_i a_i^2 \leq 1$, $(\mathbb{E}\|E_K^{(1)}\|^2)^{1/2} \le \|q\|\sigma_K D_o^{1/2} \le A_0 D_o^{1/2}$.

Further we bound $E_K^{(3+)}$
 via the per-token root-mean-square bound of
Lemma~\ref{lem:softmax-tail}. By the triangle inequality in $L^2$,
$\mathbb{E}\|E_K^{(3+)}\|
\le \max_i\|u_i\|\sum_i \mathbb{E}|\textnormal{r}_i| \leq  7 D_o^{1/2}A_0^3$, hence
\begin{equation}\label{eq:EK3-norm}
\bigl(\mathbb{E}\bigl[\|E_K^{(3+)}\|^2\bigr]\bigr)^{1/2} \le 7D_o^{1/2}\,A_0^3.
\end{equation}
Writing $E_K = E_K^{(1)} + E_K^{(2)} + E_K^{(3+)}$, the squared norm expands as
\begin{equation}\label{eq:EK-sq-expand}
\mathbb{E}\bigl[\|E_K\|^2\bigr]
= \mathbb{E}\bigl[\|E_K^{(1)}\|^2\bigr]
\;+\; \underbrace{2\,\mathbb{E}\bigl[E_K^{(1)} E_K^{(2)\top}\bigr]}_{(\star)}
\;+\; R_K,
\end{equation}
\begin{align}\label{eq:RK-def}
R_K
:=& \mathbb{E}\bigl[\|E_K^{(2)}\|^2\bigr]
+ 2\,\mathbb{E}\bigl[E_K^{(1)} E_K^{(3+)\top}\bigr]\\
&+ 2\,\mathbb{E}\bigl[E_K^{(2)} E_K^{(3+)\top}\bigr]
+ \mathbb{E}\bigl[\|E_K^{(3+)}\|^2\bigr]. \notag
\end{align}
The three factors have $L^2$ norms of orders $A_0$, $A_0^2$, and $A_0^3$ respectively:
$(\mathbb{E}\|E_K^{(1)}\|^2)^{1/2} \le A_0 D_o^{1/2}$
by~\eqref{eq:EK-final},
$(\mathbb{E}\|E_K^{(2)}\|^2)^{1/2} \le \sum_i 2a_i A_0^2\|u_i\| \le 2A_0^2 D_o^{1/2}$ via the
pathwise bound $|\delta\textnormal{a}_i^{(2)}| \le 2a_i A_0^2$ and $\sum_i a_i = 1$ and
$(\mathbb{E}\|E_K^{(3+)}\|^2)^{1/2} \le 7A_0^3 D_o^{1/2}$ by~\eqref{eq:EK3-norm}. By
Cauchy-Schwarz, each cross term in~\eqref{eq:RK-def} is bounded by the product of the
corresponding norms, so term by term:
\begin{itemize}
\item $\mathbb{E}[\|E_K^{(2)}\|^2] \le 4\,D_o\,A_0^4$;
\item $\bigl|2\mathbb{E}[E_K^{(1)} E_K^{(3+)\top}]\bigr|
      \le 2\bigl(A_0 D_o^{1/2}\bigr)\bigl(7A_0^3 D_o^{1/2}\bigr)= 14D_oA_0^4$
\item $\bigl|2\,\mathbb{E}[E_K^{(2)} E_K^{(3+)\top}]\bigr| \le 28\,D_o\,A_0^5$
\item  $\mathbb{E}[\|E_K^{(3+)}\|^2] \le 49\,D_o\,A_0^6$
\end{itemize}
Combining the four bounds,
\begin{equation*}
|R_K| \;\le\; C\, D_o\, A_0^4 :
\end{equation*}
for $A_0 \le 1$ directly, absorbing $A_0^5 \le A_0^4$ and $A_0^6 \le A_0^4$; for $A_0 > 1$ via the trivial estimate
$\|E_K\| \le \sum_i |\delta\textnormal{a}_i|\,\|u_i\| \le 2 D_o^{1/2}$
(since $\sum_i|\delta\textnormal{a}_i| \le \sum_i(\hat{\textnormal{a}}_i + a_i) = 2$), which together with the norm bounds
above controls every term of~\eqref{eq:RK-def} by $C\,D_o\,A_0^4$ in that regime.
Now we evaluate the cross-term $(\star)$, since $E_K^{(1)} = \sum_i \delta\textnormal{a}_i^{(1)} u_i$ is linear and
$E_K^{(2)} = \sum_i \delta\textnormal{a}_i^{(2)} u_i$ quadratic in $\{\alpha_\ell\}$, the scalar
$E_K^{(1)} E_K^{(2)\top}$ is a homogeneous cubic in $\{\alpha_\ell\}$, and its expectation reduces to
third-order moments $\mathbb{E}[\alpha_p\alpha_q\alpha_r]$. Each such moment falls into one of
three exhaustive cases, all vanishing under Assumption~\ref{ass:white-noise}. Three distinct indices are $\mathbb{E}[\alpha_p \alpha_q \alpha_r]=0$ by mutual independence and zero mean,
one repeated index $\mathbb{E}[\alpha_p^2 \alpha_q]$ with $p \neq q$ factors as
      $\mathbb{E}[\alpha_p^2]\,\mathbb{E}[\alpha_q] = 0$ also by independence and zero mean.
A single cubed index, $\mathbb{E}[\alpha_p^3]$ vanishes by symmetry of $\alpha_p$ (established after~\eqref{eq:alpha-sq}). Hence
\begin{equation}\label{eq:cross-vanish}
(\star) \;=\; 2\,\mathbb{E}\bigl[E_K^{(1)} E_K^{(2)\top}\bigr] \;=\; 0,
\end{equation}
so the second-order correction does not pollute the leading $\mathcal{O}(\|q\|^2 D_o\,\sigma_K^2)$
distortion, and all second-order and higher contributions are absorbed into the remainder.
 
Together with~\eqref{eq:cross-vanish}, the expansion~\eqref{eq:EK-sq-expand} reads
$\mathbb{E}[\|E_K\|^2] = D_K + R_K$ with $|R_K| \le C\,D_o\,A_0^4$.

\medskip
\noindent\textbf{Step 4: The value distortion term $\mathbb{E}[\|E_V\|^2]$.}
The value term is exact and requires no expansion, as $E_V = \sum_i a_i\, \delta\textnormal{v}_i\, \mathbf{W}_O$ with
deterministic $a_i$. Thus
\[
\mathbb{E}\bigl[\|E_V\|^2\bigr]
= \sum_{i,\ell} a_i a_\ell\, \mathbb{E}\bigl[(\delta\textnormal{v}_i \mathbf{W}_O) (\delta\textnormal{v}_\ell \mathbf{W}_O)^\top \bigr].
\]
By token independence and $\mathbb{E}[\delta\textnormal{v}_i] = 0$, the $i \neq \ell$ terms vanish, leaving the
diagonal. Expanding $\delta\textnormal{v}_i \mathbf{W}_O = \sum_c \delta\textnormal{v}_{i c}\,(\mathbf{W}_O)_{c,:}$ and using channel
independence $\mathbb{E}[\delta\textnormal{v}_{i c}\,\delta\textnormal{v}_{i c'}] = \delta_{cc'}\,\mathbb{E}[(\delta\textnormal{v}_{i c})^2]$,
\begin{align*}
\mathbb{E}\bigl[\|E_V\|^2\bigr]
&= \sum_i a_i^2\,\mathbb{E}\Bigl[\bigl\|\textstyle\sum_c \delta\textnormal{v}_{i c}(\mathbf{W}_O)_{c,:}\bigr\|^2\Bigr]\\
&= \sum_{i}\sum_{c} a_i^2\, \mathbb{E}\bigl[(\delta\textnormal{v}_{i c})^2\bigr]\, \|\mathbf{W}_{O_c}\|_F^2,
\end{align*}
where $\mathbf{W}_{O_c} := (\mathbf{W}_O)_{c,:}$ is the $c$-th row of $\mathbf{W}_O$. This is precisely $D_V$.

\medskip
\noindent\textbf{Conclusion:}
Collecting Steps 2--4,
\begin{align*}
\mathbb{E}\bigl[\|o - \hat{\textnormal{o}}\|^2\bigr]
&= \underbrace{\mathbb{E}\bigl[\|E_K^{(1)}\|^2\bigr]}_{D_K}
+ \underbrace{\mathbb{E}\bigl[\|E_V\|^2\bigr]}_{D_V}
+ \underbrace{R_2}_{\text{2(c)+2(d)}}
  + \underbrace{R_K}_{\text{Step 3}},
\end{align*}
with $|R| = |R_2+R_K| \le C\|q\|^2\|\mathbf{W}_O\|_F^2\sigma_K^2\sigma_V^2+ C\bigl(A_0^3\|\mathbf{W}_O\|_F^2\sigma_V^2 + D_oA_0^4\bigr).$
\end{proof}
For the above proof, we need a bound on the remainder. For this, we first record the derivatives of the softmax along the ray used in the proof.
\begin{remark}[Softmax derivatives along the ray]
\label{rem:softmax-expansion}
We record the derivatives of the softmax along the ray $\theta \mapsto s + \theta\alpha$,
$\theta \in [0,1]$, from which the expansion~\eqref{eq:softmax-2nd} follows by Taylor's theorem.
With $a = \mathsf{softmax}(s)$, write
$\tilde{\textnormal{a}}_\ell(\theta) := \mathsf{softmax}(s+\theta\alpha)_\ell = a_\ell e^{\theta\alpha_\ell}/\mathsf{Z}(\theta)$,
$\mathsf{Z}(\theta) := \sum_\ell a_\ell e^{\theta\alpha_\ell}$, and
$
\bar\alpha_\theta := \sum_\ell \tilde{\textnormal{a}}_\ell(\theta)\,\alpha_\ell,$
$\widetilde{m}_2(\theta) := \sum_\ell \tilde{\textnormal{a}}_\ell(\theta)(\alpha_\ell-\bar\alpha_\theta)^2$,
$\widetilde{m}_3(\theta) := \sum_\ell \tilde{\textnormal{a}}_\ell(\theta)(\alpha_\ell-\bar\alpha_\theta)^3$
for the attention-weighted mean, variance, and third central moment of the perturbations under
the intermediate weights. Then $\mathsf{g}(\theta) := \tilde{\textnormal{a}}_i(\theta)$ satisfies
\begin{align}
\mathsf{g}'(\theta) &= \tilde{\textnormal{a}}_i(\theta)\,(\alpha_i - \bar\alpha_\theta), \nonumber\\
\mathsf{g}''(\theta) &= \tilde{\textnormal{a}}_i(\theta)\bigl[(\alpha_i-\bar\alpha_\theta)^2 - \widetilde{m}_2(\theta)\bigr], \nonumber\\
\mathsf{g}'''(\theta) &= \tilde{\textnormal{a}}_i(\theta)\Bigl[(\alpha_i-\bar\alpha_\theta)^3
- 3(\alpha_i-\bar\alpha_\theta)\,\widetilde{m}_2(\theta) - \widetilde{m}_3(\theta)\Bigr].
\label{eq:g3-ray}
\end{align}
Indeed, differentiating $\log\tilde{\textnormal{a}}_\ell(\theta) = \log a_\ell + \theta\alpha_\ell - \log \mathsf{Z}(\theta)$ with
$\frac{d}{d\theta}\log \mathsf{Z}(\theta) = \bar\alpha_\theta$ gives
$\tilde{\textnormal{a}}_\ell'(\theta) = \tilde{\textnormal{a}}_\ell(\theta)(\alpha_\ell - \bar\alpha_\theta)$; hence, using
$\sum_\ell \tilde{\textnormal{a}}_\ell(\theta)(\alpha_\ell - \bar\alpha_\theta) = 0$,
\begin{align*}
\bar\alpha_\theta' &= \sum_\ell \tilde{\textnormal{a}}_\ell'(\theta)\,\alpha_\ell = \widetilde{m}_2(\theta),\\
\widetilde{m}_2'(\theta)
&= \sum_\ell \tilde{\textnormal{a}}_\ell'(\theta)(\alpha_\ell-\bar\alpha_\theta)^2
  - 2\,\bar\alpha_\theta'\sum_\ell \tilde{\textnormal{a}}_\ell(\theta)(\alpha_\ell-\bar\alpha_\theta)\\
&= \widetilde{m}_3(\theta) ,
\end{align*}
and the three formulas follow the chain rule.

At $\theta = 0$ the weights reduce to the unperturbed $a$, so
$\bar\alpha_0 = \bar\alpha$, $\widetilde{m}_2(0) = \sum_{\ell} a_\ell\bigl(\alpha_\ell - \bar\alpha\bigr)^2$, and
$\mathsf{g}'(0) = a_i(\alpha_i - \bar\alpha) = \delta\textnormal{a}_i^{(1)},
\tfrac12\,\mathsf{g}''(0) = \tfrac12\,a_i\Bigl[(\alpha_i-\bar\alpha)^2 - \sum_{\ell} a_\ell\bigl(\alpha_\ell - \bar\alpha\bigr)^2\Bigr]
= \delta\textnormal{a}_i^{(2)}$
the first- and second-order terms of~\eqref{eq:softmax-2nd} are exactly the first two Taylor
coefficients of $\mathsf{g}$ at $\theta = 0$, and $\textnormal{r}_i$ is the exact second-order Taylor remainder of $\mathsf{g}$,
as used in Lemma~\ref{lem:softmax-tail}. Summing over $i$ gives
$\sum_i \delta\textnormal{a}_i^{(1)} = \sum_i \delta\textnormal{a}_i^{(2)} = 0$, consistent with~\eqref{eq:zero-sum}.
\end{remark}
Given the above remark, let us now state and prove the following lemma, providing the softmax tail bound.
\begin{lemma}[Softmax tail bound]\label{lem:softmax-tail}
Let $\textnormal{r}_i := \delta\textnormal{a}_i - \delta\textnormal{a}_i^{(1)} - \delta\textnormal{a}_i^{(2)}$ denote the exact remainder of the
softmax response past second order, as in~\eqref{eq:softmax-2nd}. Under
Assumptions~\ref{ass:white-noise} and~\ref{ass:bounded}, for every token $i$,
\begin{equation}\label{eq:tail-as}
|\textnormal{r}_i| \le 7\,A_0^{3} \quad\text{for every token } i
\text{ and }
\sum_i |\textnormal{r}_i| \le 7\,A_0^{3}
\end{equation}
and consequently, for any deterministic row vectors $\{u_i\}$,
\begin{align}\label{eq:tail-L1}
\mathbb{E}\bigl[|\textnormal{r}_i|\bigr] &\le 7\,A_0^{3}, \
\sum_i \mathbb{E}\bigl[\textnormal{r}_i^2\bigr] \le 49\,A_0^{6},\\
\Bigl\|\sum_i \textnormal{r}_i\,u_i\Bigr\| &\le 7\,\max_i\|u_i\|\,A_0^{3}\ \text{a.s.}. \notag
\end{align}
\end{lemma}
\begin{proof}[Proof of Lemma \ref{lem:softmax-tail}]
Fix a token $i$ and set, for $\theta\in[0,1]$,
\begin{align*}
\mathsf{g}(\theta) &:= \mathsf{softmax}(s+\theta\alpha)_i,\\
\tilde{\textnormal{a}}_\ell(\theta) &:= \mathsf{softmax}(s+\theta\alpha)_\ell
= \frac{a_\ell e^{\theta\alpha_\ell}}{\mathsf{Z}(\theta)}, \
\mathsf{Z}(\theta) := \sum_\ell a_\ell e^{\theta\alpha_\ell},
\end{align*}
and write $\bar\alpha_\theta := \sum_\ell \tilde{\textnormal{a}}_\ell(\theta)\alpha_\ell$,
$\widetilde{m}_2(\theta) := \sum_\ell \tilde{\textnormal{a}}_\ell(\theta)(\alpha_\ell-\bar\alpha_\theta)^2$,
$\widetilde{m}_3(\theta) := \sum_\ell \tilde{\textnormal{a}}_\ell(\theta)(\alpha_\ell-\bar\alpha_\theta)^3$, the
attention-weighted mean, variance, and third central moment of the perturbations under the
intermediate weights. At $\theta=0$ these reduce to $\bar\alpha$ and $\widetilde{m}_2(0)$. By the ray derivatives recorded in
Remark~\ref{rem:softmax-expansion}, equation~\eqref{eq:g3-ray},
\[
\mathsf{g}'''(\theta) = \tilde{\textnormal{a}}_i(\theta)\Bigl[(\alpha_i-\bar\alpha_\theta)^3
 - 3(\alpha_i-\bar\alpha_\theta)\,\widetilde{m}_2(\theta) - \widetilde{m}_3(\theta)\Bigr].
\]
Since
$\mathsf{g}\in C^3([0,1])$ with $\mathsf{g}(0)=a_i$, $\mathsf{g}(1)=\hat{\textnormal{a}}_i$, $\mathsf{g}'(0)=\delta\textnormal{a}_i^{(1)}$, and
$\tfrac12 \mathsf{g}''(0)=\delta\textnormal{a}_i^{(2)}$, Taylor's theorem with integral remainder (see, e.g.,
\cite{zorich2004mathematical1}) gives the exact identity

\begin{equation}\label{eq:int-remainder}
\textnormal{r}_i \;=\; \frac12\int_0^1 (1-\theta)^2\, \mathsf{g}'''(\theta)\,d\theta .
\end{equation}

Now bound $\mathsf{g}'''$ pathwise.  $|\alpha_\ell| = |q^\top \delta\textnormal{k}_\ell| \le \sum_c |q_c|\,|\delta\textnormal{k}_{\ell c}| \le A_0$ almost surely for every
$\ell$. Both $\bar\alpha_\theta$ and $\bar\alpha$ are convex
combinations of the $\alpha_\ell$, so $|\bar\alpha_\theta| \le A_0$ and $|\bar\alpha| \le A_0$, hence
$|\alpha_i - \bar\alpha_\theta| \le 2A_0$,
$\widetilde{m}_2(\theta) \le (2A_0)^2 = 4A_0^2$,
$|\widetilde{m}_3(\theta)| \le (2A_0)^3 = 8A_0^3$,
the last two because $\widetilde{m}_2(\theta)$ and $\widetilde{m}_3(\theta)$ are convex
combinations of terms bounded by $4A_0^2$ and $8A_0^3$ respectively. The bracket in $\mathsf{g}'''(\theta)$
is therefore at most $8A_0^3 + 3\cdot 2A_0\cdot 4A_0^2 + 8A_0^3 = 40\,A_0^3$, uniformly in
$\theta$ and in the token. For the prefactor, $\tilde{\textnormal{a}}(\theta)$ is a probability vector, so $\tilde{\textnormal{a}}_i(\theta) \le 1$ per token and
$\sum_i \tilde{\textnormal{a}}_i(\theta) = 1$ in aggregate. Writing $\mathsf{g}_i$ for the response of token $i$, this
gives, almost surely and uniformly in $\theta\in[0,1]$,
\[
|\mathsf{g}_i'''(\theta)| \le 40 A_0^3
\text{ and }
\sum_i |\mathsf{g}_i'''(\theta)| \le 40\,A_0^3\sum_i \tilde{\textnormal{a}}_i(\theta) = 40A_0^3 .
\]
Integrating~\eqref{eq:int-remainder} with $\int_0^1(1-\theta)^2 d\theta = \tfrac13$ yields both
parts of~\eqref{eq:tail-as}: $|\textnormal{r}_i| \le \tfrac{20}{3}A_0^3 \le 7A_0^3$, and
$\sum_i |\textnormal{r}_i| \le \tfrac12\int_0^1(1-\theta)^2 \sum_i|\mathsf{g}_i'''(\theta)|\,d\theta \le 7A_0^3$. The
consequences~\eqref{eq:tail-L1} follow: an almost-sure bound dominates the mean
$\sum_i \mathbb{E}[\textnormal{r}_i^2] \le \mathbb{E}\bigl[\max_i|\textnormal{r}_i| \cdot \sum_i|\textnormal{r}_i|\bigr] \le 49A_0^6$
and $\|\sum_i \textnormal{r}_i u_i\| \le \max_i\|u_i\|\sum_i|\textnormal{r}_i|$. Assumption~\ref{ass:bounded} enters only
through the pathwise bound $|\alpha_\ell| \le A_0$.
\end{proof}

\subsection{Derivation of AATC via Waterfilling}
\label{sec:derivation-aatc-waterfillling-appendix}
Next, we derive the distortion of AATC given the transform from Section \ref{subsec:whitening}. In the decorrelated basis the hypotheses of
Theorem~\ref{thm:distortion} holds by construction: each transformed channel is
quantized independently by its own non-overloading $b_{\ell,c}$-bit uniform
quantizer, so the errors $\delta H_{ic}=\delta \hat H_{ic}-H_{ic}$ are zero-mean, symmetric, independent
across tokens and channels, and bounded, with
\begin{equation}\label{eq:noise-model}
\mathbb{E}\bigl[(\delta H_{ic})^2\bigr]
= \frac{R_{\ell,c}^2}{12\cdot 2^{2 b_{\ell,c}}}
= \gamma\,(\sigma_{\ell,c})^2\, 2^{-2 b_{\ell,c}} ,
\end{equation}
where $\sigma_{\ell,c}$ is the standard deviation of transformed channel $c$
(equal to $\sqrt{S_c}$ by the whitening), and $\gamma$ absorbs the
range-to-deviation ratio $R_{\ell,c} \propto \sigma_{\ell,c}$. Since $\gamma$ is common to
all channels, it drops out of the allocation. Instantiating
Theorem~\ref{thm:distortion} in this basis  (effective query
$\tilde q := \mathbf{B}^K q$, output metric
$\widetilde M := \mathbf{B}^V \mathbf{W}_O \mathbf{W}_O^\top (\mathbf{B}^V)^\top$)
gives $D_K = \sum_{i,c} a_i^2\|u_i\|^2\,\tilde q_c^{\,2}\,
\mathbb{E}[(\delta H_{ic})^2]$ and
$D_V = \sum_{i,c} a_i^2\,\mathbb{E}[(\delta H_{ic})^2]\,\widetilde M_{cc}$.
Since attention weights are unavailable at compression time, we allocate
against the attention-agnostic surrogate obtained by dropping the
token-dependent factors and averaging $\tilde q_c^{\,2}$ over calibration
queries:
\begin{align}\label{eq:key-distortion-channel}
    D_K^{\ell} &=  \sum_{c=1}^{d_k}  w_{\ell,c}^{(K)} (\sigma_{\ell,c}^{(K)})^2 \cdot 2^{-2b_{\ell,c}},\\
    \notag
    \text{  with  } w_{\ell,c}^{(K)} &= \frac{1}{T}  \sum_{i=1}^T ( (\mathbf{B}^{K}_{\ell,c})^{\top} q_{( i,\ell,c)})^2
\end{align}
a per-channel weight with $i$ being the tokens from the calibration set and $T$ the size of the calibration set.
The weight combines the query alignment and the reconstruction projection. $\sigma_{\ell,c}^{(K)}$ is the variance of the channel and $2^{-2b_{\ell,c}}$ captures the distortion at $b_{\ell,c}$ bits.
\begin{remark}
    For grouped query attention in which multiple query heads attend to one key head we need to sum over all the queries that attend to this head in $w_{\ell,c}^{(K)}$, resulting in  $\ w_{\ell,c}^{(K)} = \frac{1}{T} \sum_{q_{head}} \sum_{i=1}^T ( (\mathbf{B}^{K}_{\ell,c})^{\top} q_{(q_{head}, i,\ell,c)})^2$
\end{remark}
The value distortion takes the analogous form
\begin{align}\label{eq:value-distortion-channel}
    D_V^{\ell} =  \sum_{c=1}^{d_v}   w_{\ell,c}^{(V)} (\sigma_{\ell,c}^{(V)})^2 \cdot  2^{-2b_{\ell,c}}
\end{align}
 with  $w_{\ell,c}^{(V)} = \lVert (\mathbf{B}^V \mathbf{W}_O)_{\ell,c}^\top \rVert_F^2$.
This is the formula from Theorem \ref{thm:distortion} with the transform plugged into it.
The above defines the distortions per-layer. Overall, we can obtain the global distortions as
\[
D^{global}_K= \sum_{\ell=1}^L D_K^{\ell} ; \;  \; D^{global}_V= \sum_{\ell=1}^L D_V^{\ell}.
\]
Given the above, we can formulate the waterfilling problem as in Section \ref{subsubsec:bit-allocation}.

\subsection{Mapping Related Work onto the Distortion Formula}
\label{app:mapping-distortion}
\subsubsection{Token-Level Attention Weight}
\label{subsec:attention-weight}
The factor $a_i^2$ captures how strongly token $i$ contributes to the attention output. Tokens with vanishing attention weights contribute negligibly to $D$ no matter how their keys and values are quantized, so the natural way to exploit this is to drop such tokens entirely rather than spend bits on them. In the factorization of Eq.~\ref{eq:factorized-distortion}, $a_i^2$ sits inside the token-dependent factor, so methods that exploit it operate orthogonally to the channel-wise allocation we develop here. Token eviction methods take exactly this route.

The earliest of them use attention magnitude as the eviction criterion: StreamingLLM \cite{xiao2023efficient} preserves the first tokens together with a sliding window of recent tokens, H2O \cite{zhang2023h2o} retains "heavy hitter" tokens identified by accumulated attention scores, SnapKV \cite{li2024snapkv} selects which tokens to keep during prefill from recent attention patterns. A more recent line of work refines this criterion with value-side information: VATP \cite{guo2024attention} take into account attention scores and value $\ell_1$-norms for token eviction, CriticalKV \cite{feng2025identify} introduces a method evicting tokens based on a projected value norm $\lVert V_iW_O \rVert_1$ combined with the attention scores and Expected Attention \cite{devoto2025expected} uses $a_i \|W_O v_i\|$ as each KV pair's residual-stream contribution by approximating $a_i$ through an expected attention score. All three target the contribution magnitude $a_i \|v_i W_O\|$, the size of token 
$i$'s direct contribution to the output. This ignores the attention mass that would redistribute to the remaining tokens if $i$ were evicted. The full leading-order eviction cost includes that redistribution and takes the form $a_i \|v_i W_O - o\|$, which CAOTE \cite{goel2025caote} almost computes in closed form, will be discussed in Subsection~\ref{subsec:value-relevance}.

\subsubsection{Token-Level Output Relevance}
~\label{subsec:value-relevance}
The factor $\|v_i W_O - o\|^2$ measures how much token
$i$'s value, after the output projection, deviates from the attention output.  It arises in $D_K$
 as the sensitivity of the output to perturbations in attention allocation caused by key quantization. If $v_i W_O \approx o$, then key errors at this token are forgiving because moving attention onto or off token $i$ does not change $o$ much. Like $a_i^2$, this factor lives in the token-dependent part of the factorization (Eq.~\ref{eq:factorized-distortion}) and would naturally belong to a token-wise allocation rather than the channel-wise one we address.
The closest precedent in prior KV compression work is CAOTE~\cite{goel2025caote}, which defines token eviction error as the change in attention output upon token removal and derives in closed form that this error has the structure of a residual relative to the attention-weighted average of the value vectors. CAOTE's quantity differs from that in the attention-aware distortion formula by operating in value space rather than residual-stream space (i.e., before the output projection $W_O$). The structural insight is, however, shared. What matters is how different a token's value contribution is from the rest of the output, not just how large its raw contribution is. The other value-aware eviction methods discussed in Section~\ref{subsec:attention-weight} instead target the contribution magnitude $a_i \|v_i W_O\|$, a related quantity that never references $o$ and so misses this distinction.

\subsubsection{Query-Aware Precision}

The factor $q_c^2$ in $D_K$ motivates allocating precision to the key directions that queries read from strongly. A²ATS \cite{he2025a2ats} derives the attention score error component and uses it to motivate query-aware vector quantization. MixKVQ \cite{zhang2026mixkvq} uses query relevance for mixed-precision allocation. SQuat \cite{wang2025squat} projects the key quantization error onto the orthogonal complement of the query subspace, so that the retained error lies in directions queries do not read from; this is query-awareness through the error geometry rather than through bit allocation. OSCAR \cite{zhou2026oscar} addresses the factor through rotation: working with the full query covariance $C_Q = Q^\top Q$ rather than a per-channel energy, it shows under a frozen-error surrogate (post-rotation residual covariance) that the eigendecomposition of $C_Q$ minimizes the surrogate (their Theorem 1). OSCAR composes this eigenrotation with a Hadamard transform and bit-reversal permutation, then applies uniform INT2 in the resulting basis.
Our method takes a different route to the same goal on the key side. The rotation we use is obtained by whitening the cached activations and is not query-aware. The query signal enters through the bit allocation instead, with channels carrying more query energy receiving more bits. These two routes, query-awareness through the rotation versus through the allocation, coincide in their treatment of the query signal. Each, on its own, already directs precision toward the directions queries read from, so along that axis, applying both double-counts it. However, one could combine OSCAR's query-aware rotation with channel-wise bit allocation.

\subsubsection{Output-Projection-Aware Value Precision}

The factor $W_{O_c}$, the column norm of the output projection at value channel $c$, appears in $D_V$ as a multiplier on the value quantization error in that channel. A value channel that the output projection reads from with a large weight has its quantization error amplified at the output; a channel with a small projection weight is forgiving. The decomposition, therefore, motivates allocating value precision in proportion to $W_{O_c}$ (or a norm-based aggregate over the output dimensions). To our knowledge, no prior method exposes $W_{O_c}$ as an explicit allocation factor.
$W_O$ enters KVQuant's Fisher sensitivity only implicitly, through the squared gradient of the language-modeling loss with respect to the value activations, $J(V) = \partial L/\partial V$, where $W_O$ appears entangled with the attention scores and the rest of the downstream gradient. Our criterion instead names $W_{O_c}$ as an explicit, closed-form structural term.
OSCAR's value rotation uses the score-weighted value covariance $V^\top A^\top A V$, which contains no factor from the output projection. Our allocation criterion incorporates $W_{O_c}$ directly, as a structural term derived from the decomposition which remains fixed for every layer, rather than a quantity inferred through calibration.

\subsubsection{Per-Channel Sensitivity}

The terms $\mathbb{E}[(\delta k_{ic})^2]$ and $\mathbb{E}[(\delta v_{ic})^2]$ capture how much quantization error each channel produces under a given quantizer. KIVI \cite{liu2024kivi} addresses this by per-channel asymmetric scaling for keys and per-token scaling for values. KVQuant \cite{hooper2024kvquant} additionally employs sensitivity-weighted non-uniform quantization. PALU \cite{chang2024palu} applies SVD-based low-rank truncation. RateQuant \cite{zuo2026ratequant} is the closest to our approach. It applies rate-distortion theory to mixed-precision allocation, fitting a distortion-rate curve $D(b)=\alpha \beta^{- b}$ per component from calibration data and solving via reverse waterfilling. The substantive difference is not granularity (RateQuant operates at the (head, K-or-V) level but is granularity-agnostic in principle), but what the framework makes explicit. RateQuant's sensitivity weights are squared gradient norms and its distortion curve is a parametric model whose parameters $(\alpha,\beta)$ are fit empirically from MSE at a handful of precisions. Query-awareness, the output-projection weights $W_{O_c}$ and the attention-score weighting $a_i$ never appear. They are bundled invisibly into the fitted gradient norms and curve parameters. Our decomposition surfaces these factors explicitly and ties each to a distinct design axis, deriving in closed form the structure that RateQuant leaves implicit in its empirical fits.

\subsubsection{Synthesis (Our Method)}

The four channel-level factors in the decomposition, namely $q_c^2$ and $\mathbb{E}[(\delta k_{ic})^2]$  on the key side as well as $W_{O_c}$ and $\mathbb{E}[(\delta v_{ic})^2]$ on the value-side, together determine each channel's contribution to $D$. Our method allocates non-uniform precisions across channels using a criterion that combines these factors. The rotation is calibration-derived from the cached activations and is shared across the bit-allocation step. We do not optimize it separately for query awareness, since the bit allocation already directs precision to the channels that matter most according to the attention statistics.

Each row of the table above is addressed by an existing method through an existing mechanism. Our contribution is to address the channel-level rows \textit{jointly} through a single allocation criterion derived from the decomposition. Rather than picking one factor and optimizing for it via a single mechanism, we use the decomposition to set precisions that respond to all channel-level factors simultaneously. Our precise method is described in Section \ref{sec:method-channel-allocation}.

\subsection{Comparison of SVD and Whitening}
\label{subsec:svdequalwhitening}
KVTC \cite{staniszewski2025kv} derives its transform from the PCA of $\mathbf{K} := \mathbf{X}\mathbf{W} $, whereas we use a whitening-based SVD of the projection $\mathbf{W}$. We show that these two constructions yield the same transform. The equivalence is confined to this decorrelation step: KVTC decorrelates a global cross-layer cache and allocates on reconstruction error, while we decorrelate per layer and allocate on the output-aware distortion of Theorem \ref{thm:distortion}.
\begin{proposition}[Equivalence of Whitening-Based SVD and PCA of $\mathbf{X}\mathbf{W}$]
Let $\mathbf{X} \in \mathbb{R}^{n \times d_{\mathrm{model}}}$ be calibration activations,
$\mathbf{W} \in \mathbb{R}^{\mathrm{model} \times d_k}$ a linear projection matrix,
and define $\mathbf{K} := \mathbf{X}\mathbf{W} \in \mathbb{R}^{n \times d_k}$.
Let the empirical input covariance be
$\mathbf{\Sigma}_X := \frac{1}{n} \mathbf{X}^\top \mathbf{X}$
and assume $\mathbf{\Sigma}_X$ is symmetric positive definite.
Define the whitened weight matrix
$\mathbf{\widetilde W} := \mathbf{\Sigma}_X^{1/2} \mathbf{W}$.
Let $\mathbf{\widetilde W} = \mathbf{U} \mathbf{S} \mathbf{P}^\top $ be its singular value decomposition.
Then the following statements hold:
\begin{enumerate}
    \item The right singular vectors $\mathbf{P}$ coincide with the right singular vectors of $\mathbf{K}$.
    \item The transform coding coordinates obtained via $\mathbf{K}\mathbf{P}$ are identical to
    $\mathbf{X} \mathbf{\Sigma}_X^{-1/2} \mathbf{U} \mathbf{S}$.
\end{enumerate}
\end{proposition}

\begin{proof}
First observe that
\[
\mathbf{K}^\top \mathbf{K}
=
\mathbf{W}^\top \mathbf{X}^\top \mathbf{X} \mathbf{W}
=
n\, \mathbf{W}^\top \mathbf{\Sigma}_X \mathbf{W}.
\]
On the other hand,
$\ \mathbf{\widetilde W}^\top\mathbf{\widetilde W}
=
\mathbf{W}^\top \mathbf{\Sigma}_X \mathbf{W}.$
Therefore,
$\mathbf{K}^\top \mathbf{K} =
n\,  \mathbf{\widetilde W}^\top \mathbf{\widetilde W}.$
Since multiplication by the positive scalar $n$ does not change eigenvectors,
the right singular vectors of $\mathbf{K}$ are identical to those of $ \mathbf{\widetilde W}$.
Hence, both SVDs yield the same matrix $\mathbf{P}$ of right singular vectors.

Next, express $\mathbf{W}$ using the SVD of $ \mathbf{\widetilde W}$:
$
\mathbf{W}
=
\mathbf{\Sigma}_X^{-1/2} \mathbf{U} \mathbf{S} \mathbf{P}^\top.$
Substituting into $\mathbf{K} = \mathbf{X}\mathbf{W}$ gives
$\mathbf{K}=\mathbf{X} \mathbf{\Sigma}_X^{-1/2} \mathbf{U} \mathbf{S} \mathbf{P}^\top.$
Right-multiplying by $\mathbf{P}$ yields
$\mathbf{K} \mathbf{P}
=
\mathbf{X} \mathbf{\Sigma}_X^{-1/2} \mathbf{U} \mathbf{S}. $
Thus the transform coding coordinates obtained by
right-multiplying $\mathbf{K}$ with $\mathbf{P}$
are identical to those obtained via
$\mathbf{X} \mathbf{\Sigma}_X^{-1/2} \mathbf{U} \mathbf{S}$.
\end{proof}
\begin{remark}
First, the same holds for the values. Second, our method uses $\mathbf{U}\sqrt{\mathbf{S}}$ rather than $\mathbf{U}\mathbf{S}$, so even the transform-coding coordinates differ by the factor $\sqrt{\mathbf{S}}$, with $\mathbf{S}$ the singular values.
\end{remark}

\section{Experiment Setup}

\subsection{Experimental Setting for KV Cache Quantization Methods}
\label{subsec:compare-methods}
Let us explain in more detail the experimental settings we used for the different KV cache quantization methods appearing in Section \ref{sec:experiments}.

\paragraph{KIVI}
We use  KIVI \cite{liu2024kivi} as in its original implementation, quantizing keys channel-wise and values token-wise with recent-token window of size $128$ and a scaling group-size of $32$. We also keep the first $4$ tokens in full precision.

\paragraph{PALU}
For PALU \cite{chang2024palu}, we apply their Fisher-uniform implementation with head group size $4$ and the decompose-method whiten. For the low-rank part we establish a recent token window of size $128$ with a sliding window of $16$ and also keep the first $4$ tokens in full precision. Further, we use uniform quantization with Hadamard rotation for the low-rank part.

\paragraph{KVQuant}
We reimplemented KVQuant~\cite{hooper2024kvquant}, following their original method. As in their work, we quantize keys per-channel and values per-token, quantize keys pre-RoPE, and use their non-uniform quantization using a look-up table of $2^b$ signposts derived via Fisher-sensitivity-weighted k-means. We also implement their dense-and-sparse decomposition, which keeps the top $1\%$ of magnitude outliers in full precision. For fairness, we match the FP16 window budget of the other methods, keeping a recent window of $128$ tokens (flushed in blocks of $32$) and keep the first $4$ tokens in full precision.
\paragraph{AATC} If not stated otherwise, we keep the first $4$ tokens in full precision and set the full precision recent token window to $128$ with a sliding step $16$. The bit range allowed for bit allocation per channel is $0$ to $16$ and the allocation is done globally across all layers for keys and values as described in Section \ref{sec:method-channel-allocation}. Since the first three layers are quite important, but exhibit smaller activations and gradients and therefore smaller weights in the waterfilling problem, we protect them and give them the average amount of bits for the values and $1$ bit more than average for the keys. We always use the same number of bits for keys and values. For the global allocation, we normalize $\sigma^2$ by each layer's own $95$-th percentile to remove the activation-scale bias between the layers, since early layers have smaller activations and thus smaller $\sigma^2$ and we do the same for $w^{(K)}$

\paragraph{AATC var-only} This refers to our method, but with allocating the bits via reverse waterfilling without using $q_c^2$ and $W_{o_c}$ in the weights, therefore only using the variance of the quantization error.

\subsection{Benchmarks}
\label{subsec:benchmarks-quality}
Here is a more detailed explanation of the benchmarks we use in our experiments.
\paragraph{\texttt{LongBench} v1 \cite{bai2024longbench}} is a bilingual, multi-task benchmark for evaluating long-context understanding,
  with document lengths ranging from roughly 1k to 34k tokens. We evaluate on seven English subsets spanning four
  task categories: single-document QA (TriviaQA, Qasper), query-based summarization (QMSum, SAMSum), few-shot
  classification (TREC), and code completion (LCC, RepoBench-P). Following the standard \texttt{LongBench} evaluation
  protocol, we use F1 score for QA tasks, ROUGE-L for summarization, accuracy for classification, and
  edit-similarity for code completion, and report results for each subset individually. We use greedy decoding with $0$-shots and task-dependent maximum new tokens ranging from $ 32$ to $512$ (TriviaQA: $32$, TREC: $64$, Qasper: $128$, SAMSum: $128$, LCC: $64$, RepoBench-P: $64$, QMSum: $512$). 
  \paragraph{\texttt{RULER} \cite{hsieh2024ruler} (What's the Real Context Size of Your Long-Context Language Models?)} also evaluates long context models and tests different context lengths to check if the model still performs well on longer contexts or when its performance degrades. The context lengths are split into  4k, 8k, 16k, 32k, 64k, and 128k, we focus on 4k-32k. It includes 13 tasks across 4 categories (retrieval, multi-hop tracing, aggregation, and question answering) with configurable sequence lengths (4k to 128k tokens). For evaluation, \texttt{RULER} uses task-specific metrics. For Needle-in-a-Haystack (NIAH), those are variable tracking, common word extraction, and frequency word extraction, it checks whether all reference answers are found in the model's response. Then it calculates Score = (samples where all answers are found) / (total samples) × 100. For QA tasks, it checks if any reference answer is found in the model's response and the score is calculated as (samples where at least one answer is found) / (total samples) × 100. For evaluation, we use 500 samples per task, which is a common convention. \texttt{RULER} uses $0$-shot greedy decoding with a maximum of $512$ new tokens.

  \paragraph{\texttt{GSM8K} (Grade School Math 8K) \cite{cobbe2021training}} contains 8,500 linguistically diverse grade school math word problems that require multi-step reasoning. The problems take between 2 and 8 steps to solve and mostly involve basic algebra and arithmetic operations. The solutions are given in natural language. For evaluation, we report the percentage of answers correctly extracted from model outputs using 8-shot chain-of-thought prompting and the flexible-extract metric, which parses the final numerical answer from free-form generated text before comparison. 

\paragraph{\texttt{MMLU-Pro} \cite{wang2024mmlu}} is a challenging multiple-choice benchmark extending MMLU with up to 10 answer
  options and questions requiring deeper reasoning. We evaluate on the math and computer science subjects only,
  using 5-shot prompting with examples drawn from the validation split of the same subject. The model is asked to
  respond with a single letter (A-J). We use greedy decoding with a 32-token budget and extract the predicted answer
  via regex matching on the generated text. Accuracy is reported as the fraction of test questions answered
  correctly across both subjects. \texttt{LongBench} uses greedy decoding, allows up to $32$ new tokens and truncates the input to $4096$ tokens.

  \paragraph{\texttt{MATH-500}} \cite{lightman2023lets} is a 500-problem subset of the MATH competition
  benchmark \cite{hendrycksmath2021}, spanning seven subjects (algebra, number theory, geometry, etc.) and five difficulty levels. We evaluate using 0-shot chain-of-thought prompting, instructing the model to reason step-by-step and place its final
  answer as $\boxed{answer}$. Answers are extracted
  from the last $\boxed{answer}$ expression and graded for symbolic equivalence using math-verify, with a LaTeX
  normalization fallback. We report accuracy across all 500 problems. \texttt{MATH-500} also uses greedy decoding, $2048$ max new tokens.

\subsection{Per-Subtask \texttt{LongBench} Results}
\label{subsec:appendix-longbench}
 
Table \ref{tab:longbench-full} gives the per-subtask \texttt{LongBench} scores underlying the
averages in Table \ref{tab:mainresults}. The breakdown explains PALU's weak average on Qwen.
Its degradation is not uniform but almost entirely localized to the two
code-completion tasks (RepoBench-P $0.34$ vs.\ $0.67$ for FP16, and LCC
$0.48$ vs.\ $0.61$), while summarization is essentially lossless (QMSum $0.234$
vs.\ $0.237$). This is the characteristic signature of low-rank KV compression.
The truncated low-rank components carry the fine-grained, token-level information
that exact-match code completion and long-range retrieval depend on, whereas the
retained dominant subspace suffices for short-context needle retrieval. The same
mechanism accounts for PALU's \texttt{RULER} profile in Table \ref{tab:mainresults}, competitive
at 4k but degrading monotonically with length.

\begin{table*}[p]\centering\scriptsize\setlength{\tabcolsep}{3pt}\renewcommand{\arraystretch}{0.95}\setlength{\aboverulesep}{1pt}\setlength{\belowrulesep}{1pt}
\caption{Per-subtask \texttt{LongBench} scores at the main-table operating point: F1 for QA (TriviaQA, Qasper), ROUGE-L for summarization (SAMSum, QMSum) and query summarization, accuracy for classification (TREC), and edit-similarity for code (LCC, RepoBench-P). \textbf{Bold} = best compressed method per column, and the FP16 baseline is a reference.}
\label{tab:longbench-full}
\begin{tabular}{c l | ccccccc | c}
\toprule
 & Method & TriviaQA & Qasper & TREC & SAMSum & LCC & RepoB-P & QMSum & avg. \\
\midrule
\multirow{5}{*}{\texttt{Llama-3.1-8B-Instruct}} & Baseline & 0.917 & 0.455 & 0.725 & 0.437 & 0.634 & 0.567 & 0.254 & 0.570 \\
 & KIVI & 0.913 & 0.440 & \textbf{0.725} & \textbf{0.448} & 0.619 & 0.552 & 0.250 & 0.564 \\
 & PALU &  0.892   &  0.449  & 0.710  &       0.438   &  0.539  &  0.576  & 0.248 & 0.550 \\
 & KVQuant & 0.921 & 0.413 & 0.675 & 0.415 & 0.633 & 0.546 & 0.241 & 0.549 \\
  & AATC var & \textbf{0.922} & \textbf{0.452} & 0.715 & 0.430 & \textbf{0.634} & \textbf{0.559} & 0.251 & \textbf{0.566} \\
 & AATC (ours) & 0.919 & 0.444 & 0.690 & 0.434 & 0.623 & 0.546 & \textbf{0.252} & 0.558 \\
\midrule
\multirow{6}{*}{\texttt{Qwen-2.5-7B-Instruct}} & Baseline & 0.900 & 0.434 & 0.715 & 0.460 & 0.606 & 0.668 & 0.237 & 0.574 \\
 & KIVI & 0.597 & 0.187 & 0.568 & 0.337 & 0.283 & 0.316 & 0.168 & 0.351 \\
 & PALU & 0.876 & 0.367 & 0.650 & 0.440 & 0.477 & 0.342 & \textbf{0.234} & 0.484 \\
 & KVQuant & 0.900 & 0.409 & \textbf{0.700} & 0.449 & 0.583 & 0.640 & 0.226 & 0.558 \\
  & AATC var & 0.896 & 0.425 & 0.675 & \textbf{0.461} & 0.569 & \textbf{0.648} & 0.233 & 0.558 \\
 & AATC (ours) & \textbf{0.912} & \textbf{0.435} & \textbf{0.700} & 0.452 & \textbf{0.609} & 0.631 & 0.232 & \textbf{0.567} \\
\bottomrule\end{tabular}
\end{table*}

\begin{table*}[p]\centering\footnotesize\setlength{\tabcolsep}{4pt}\renewcommand{\arraystretch}{0.95}\setlength{\aboverulesep}{1pt}\setlength{\belowrulesep}{1pt}
\caption{AATC at the aggressive 2-bit operating point ($\approx$7$\times$ compression), against the FP16 reference. Baseline compression methods (KIVI, PALU, KVQuant), which run at 2-bit and $\approx$5 to 6$\times$, are reported in the main results table (\ref{tab:mainresults}). }
\label{tab:2bit-supp}
\begin{tabular}{c l | c | c | cccc | cccc}
\toprule
 & & & \texttt{LongBench} & \multicolumn{4}{c|}{RULER} & \multicolumn{4}{c}{\textbf{Reasoning}} \\
 & Method & Comp. & avg. & 4k & 8k & 16k & 32k & \texttt{GSM8K} & MMLU-m & MMLU-cs & MATH \\
\midrule
\multirow{3}{*}{Llama} & Baseline (FP16) & 1.00 & 0.570 & 0.859 & 0.856 & 0.828 & 0.775 & 0.789 & 0.438 & 0.456 & 0.454 \\
 & AATC var & 7.0 & 0.543 & 0.818 & 0.822 & 0.775 & 0.677 & \textbf{0.767} & 0.402 & 0.427 & \textbf{0.484} \\
 & AATC (ours) & 7.0 & \textbf{0.568} & \textbf{0.851} & \textbf{0.861} & \textbf{0.809} & \textbf{0.735} & 0.756 & \textbf{0.416} & \textbf{0.434} & 0.448 \\
\midrule
\multirow{3}{*}{Qwen} & Baseline (FP16) & 1.00 & 0.574 & 0.720 & 0.716 & 0.718 & 0.720 & 0.818 & 0.708 & 0.567 & 0.680 \\
 & AATC var & 6.9 & 0.518 & 0.688 & 0.653 & 0.640 & 0.566 & 0.795 & 0.644 & 0.510 & 0.664 \\
 & AATC (ours) & 6.9 & \textbf{0.555} & \textbf{0.703} & \textbf{0.702} & \textbf{0.704} & \textbf{0.669} & \textbf{0.802} & \textbf{0.658} & \textbf{0.542} & \textbf{0.676} \\
\bottomrule\end{tabular}
\end{table*}
 
\subsection{Robustness across different Bit Budgets}
\label{subsec:appendix-bitsweep}
 
Table \ref{tab:bitsweep-appendix} and Figure \ref{fig:bitsweep} sweep AATC across KV bit averages from 4-bit down to 1-bit. AATC degrades gracefully. The \texttt{LongBench}
average remains within the noise of FP16 from 4-bit down to 2-bit, and \texttt{GSM8K} declines
smoothly, with a pronounced drop appearing only at 1.5-bit and below. At the
extreme 1-bit setting, AATC remains substantially more robust than KIVI (\texttt{LongBench}
$0.33$ vs.\ $0.20$, \texttt{GSM8K} $0.50$ vs.\ $0.03$ on Llama), indicating that the
transform-and-allocation design fails far more gracefully than fixed-precision
quantization once the bit budget is exhausted.

Table \ref{tab:2bit-supp} isolates AATC's 2-bit operating point ($\approx\!7\times$)
against the FP16 reference. The 2-bit quantization baselines (KIVI and KVQuant), which run at
$\approx\!5$ to $6\times$, are reported in Table \ref{tab:mainresults}. AATC stays close
to FP16 on Qwen across all context lengths, while AATC var-only falls off on
long-context \texttt{RULER}, consistent with Figure \ref{fig:divergence}.

\begin{figure}
    \centering
    \includegraphics[width=1.0\linewidth]{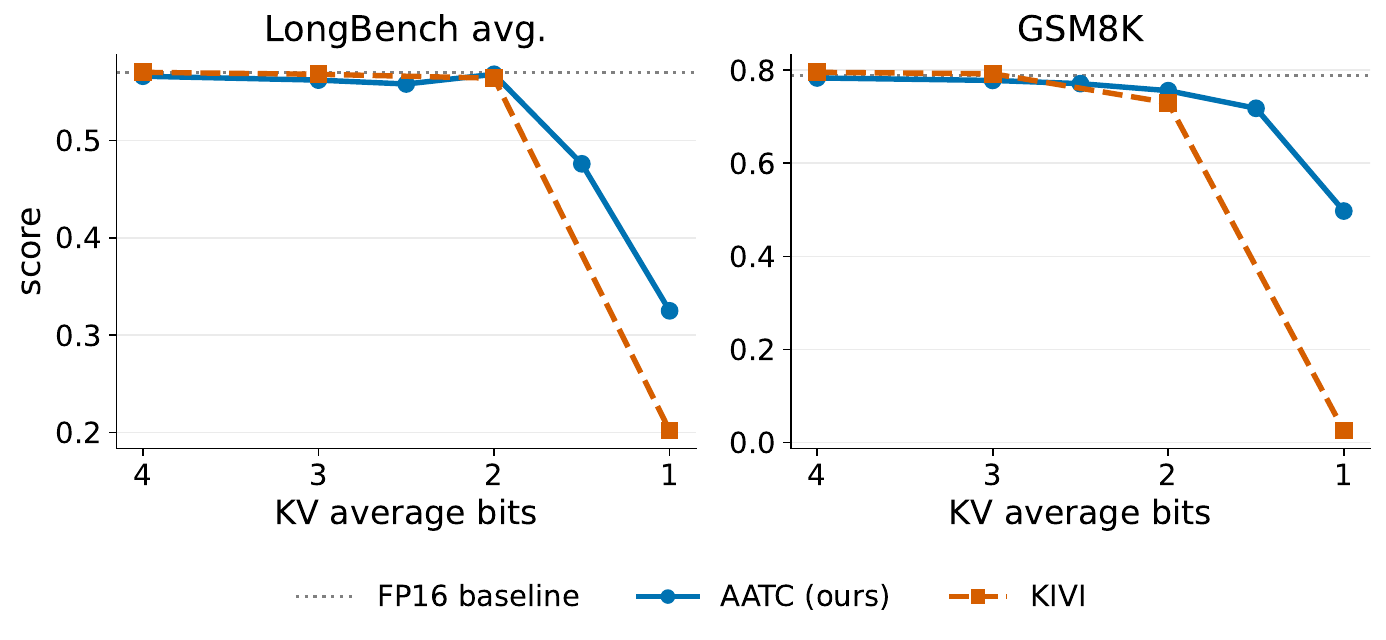}
    \caption{\texttt{LongBench} average and \texttt{GSM8K} versus KV averages bits on \texttt{Llama-3.1-8B-Instruct}.}
    \label{fig:bitsweep}
\end{figure}

\begin{table*}[p]\centering\scriptsize\setlength{\tabcolsep}{3pt}\renewcommand{\arraystretch}{0.95}\setlength{\aboverulesep}{1pt}\setlength{\belowrulesep}{1pt}
\caption{\texttt{LongBench} per-subtask scores and \texttt{GSM8K} across KV multiple bit-budgets. Metrics: F1 (TriviaQA, Qasper), accuracy (TREC), ROUGE-L (SAMSum, QMSum), edit-similarity (LCC, RepoBench-P). LB avg.\ is the seven-task mean.}
\label{tab:bitsweep-appendix}
\begin{tabular}{c l | c| ccccccc | c | c}
\toprule
 & Config &Comp & TriviaQA & Qasper & TREC & SAMSum & LCC & RepoB-P & QMSum & LB avg. & \texttt{GSM8K} \\
\midrule
\multirow{15}{*}{Llama-3.1-8B--Instruct} & Baseline FP16 & 1.0 &0.917 & 0.455 & 0.725 & 0.437 & 0.634 & 0.567 & 0.254 & 0.570 & 0.789\\
 & KIVI 4-bit & 3.1&0.917 & 0.454 & 0.725 & 0.440 & 0.634 & 0.568 & 0.250 & 0.570 & 0.795 \\
 & KVQuant 4-bit & 3.4& 0.914 & 0.447 & 0.725 & 0.426 & 0.631 & 0.560 & 0.256 & 0.566 & 0.779 \\
 & AATC 4-bit &3.8 &0.916 & 0.446 & 0.725 & 0.436 & 0.624 & 0.565 & 0.253 & 0.566 & 0.783 \\
 & KIVI 3-bit & 3.8 &0.914 & 0.450 & 0.725 & 0.435 & 0.634 & 0.564 & 0.254 & 0.568 & 0.792 \\
 & KVQuant 3-bit & 4.3& 0.921 & 0.452 & 0.725 & 0.433 & 0.631 & 0.559 & 0.250 & 0.567 & 0.776\\
 & AATC 3-bit & 4.9 &0.920 & 0.458 & 0.705 & 0.439 & 0.613 & 0.549 & 0.251 & 0.562 & 0.778 \\
 & AATC-var 2-bit & 7.0 & 0.897 & 0.435 & 0.650 & 0.436 & 0.619 & 0.519 & 0.243 & 0.543 & 0.767 \\
 & AATC 2-bit & 7.0 &0.917 & 0.456 & 0.720 & 0.438 & 0.630 & 0.570 & 0.248 & 0.568 & 0.756 \\
 & AATC 1.5-bit &9.1 & 0.874 & 0.336 & 0.630 & 0.434 & 0.480 & 0.358 & 0.223 & 0.476 & 0.718 \\
 & KIVI 1-bit & 7.1 &0.256 & 0.059 & 0.368 & 0.114 & 0.257 & 0.241 & 0.122 & 0.202 & 0.026 \\
 & AATC 1-bit &12.4 &0.614 & 0.111 & 0.265 & 0.289 & 0.445 & 0.375 & 0.174 & 0.325 & 0.497 \\
\midrule
\multirow{13}{*}{Qwen-2.5-7B--Instruct} & Baseline FP16 &1.0 &0.900 & 0.434 & 0.715 & 0.460 & 0.606 & 0.668 & 0.237 & 0.574 & 0.818 \\
 & KIVI 4-bit &3.1 &0.734 & 0.291 & 0.665 & 0.399 & 0.451 & 0.518 & 0.195 & 0.465 & 0.131 \\
 & KVQuant 4-bit &3.4 &0.897 & 0.428 & 0.710 & 0.465 & 0.605 & 0.667 & 0.236 & 0.573 & 0.820 \\
 & AATC-var 4-bit & 3.8&0.891 & 0.427 & 0.735 & 0.459 & 0.604 & 0.661 & 0.238 & 0.574 & 0.820 \\
 & AATC 4-bit &3.8 &0.894 & 0.434 & 0.700 & 0.459 & 0.612 & 0.664 & 0.232 & 0.571 & 0.825 \\
 & KIVI 3-bit & 3.8 &0.686 & 0.233 & 0.620 & 0.374 & 0.334 & 0.380 & 0.188 & 0.402 & 0.182 \\
 & KVQuant 3-bit &4.3 &0.921 & 0.452 & 0.725 & 0.433 & 0.631 & 0.558 & 0.250 & 0.567 & 0.816 \\
 & AATC-var 3-bit &4.9 &0.907 & 0.424 & 0.705 & 0.464 & 0.596 & 0.652 & 0.233 & 0.569 & 0.832 \\
 & AATC 3-bit &4.9 &0.888 & 0.433 & 0.705 & 0.457 & 0.599 & 0.667 & 0.238 & 0.570 & 0.818 \\
 & AATC-var 2-bit &6.9 &0.871 & 0.365 & 0.620 & 0.445 & 0.500 & 0.603 & 0.222 & 0.518 & 0.795 \\
 & AATC 2-bit &6.9 &0.898 & 0.386 & 0.705 & 0.456 & 0.597 & 0.612 & 0.230 & 0.555 & 0.802 \\
\bottomrule\end{tabular}
\end{table*}

\begin{table*}[p]\centering\footnotesize\setlength{\tabcolsep}{4pt}\renewcommand{\arraystretch}{0.95}\setlength{\aboverulesep}{1pt}\setlength{\belowrulesep}{1.0pt}
\caption{Ablation of AATC's bit-allocation and calibration choices (2-bit, \texttt{Llama-3.1-8B-Instruct}, with \texttt{LongBench} per-subtask, LB avg.\ = seven-task mean and \texttt{GSM8K}). }
\label{tab:longbench-AATC-calib-ablation}
\begin{tabular}{l|ccccccc|c|c}
\toprule
AATC variant (2-bit, \texttt{Llama-3.1-8B-Instruct}) & TriviaQA & Qasper & TREC & SAMSum & LCC & RepoB-P & QMSum & LB avg. & \texttt{GSM8K} \\
\midrule
AATC 2-bit (standard) & \textbf{0.917} & \textbf{0.456} & \textbf{0.720} & 0.438 & \textbf{0.630} & \textbf{0.570} & 0.248 & \textbf{0.568} & 0.751 \\
\multicolumn{10}{l}{\textit{Allocation}} \\
\quad no $q_c^2$ in allocation & 0.910 & 0.410 & 0.685 & 0.444 & 0.545 & 0.473 & 0.248 & 0.531 & 0.752 \\
\quad no $W_O$ in allocation & 0.916 & 0.452 & \textbf{0.720} & 0.430 & 0.602 & 0.551 & \textbf{0.251} & 0.560 & \textbf{0.763} \\
\quad per-layer allocation & 0.913 & 0.430 & 0.715 & 0.445 & 0.583 & 0.533 & 0.249 & 0.553 & 0.760 \\
\quad global allocation & 0.847 & 0.337 & 0.595 & 0.432 & 0.475 & 0.462 & 0.231 & 0.483 & 0.740 \\
\multicolumn{10}{l}{\textit{Calibration set}} \\
\quad calib.\ OpenR1-Math only & 0.915 & 0.449 & 0.710 & \textbf{0.454} & 0.600 & 0.540 & 0.248 & 0.559 & 0.761 \\
\quad calib.\ FineWeb only & 0.895 & 0.437 & 0.690 & 0.446 & 0.564 & 0.540 & 0.248 & 0.546 & 0.754 \\
\bottomrule\end{tabular}
\end{table*}
 
\subsection{Ablation: Allocation and Calibration}
\label{subsec:appendix-ablation}
 
Table \ref{tab:longbench-AATC-calib-ablation} ablates the two design axes of AATC at
2-bit on Llama. The bit-allocation strategy is the dominant factor. Replacing the
adaptive allocation with a single global allocation costs $8.6$ \texttt{LongBench} points
($0.568$ to $0.483$), and removing the query-norm term ($q_c^2$) costs $3.8$
points. In contrast, omitting
the output-projection ($W_O$) term ($-0.8$) and using a per-layer rather than
per-channel allocation ($-1.6$) does not drop performance significantly. For the calibration-corpus we cannot see a significant drop for only using OpenR1-Math $-0.9$, for FineWeb the drop is slightly larger, which one could interpret as reasoning text being important for calibration. \texttt{GSM8K} is
not much affected by any variants, so the ablation is
resolved on \texttt{LongBench}. Consistent with Section \ref{subsec:appendix-longbench}, the
damage from the significant ablations concentrates on the precision-sensitive code
and retrieval subtasks (global allocation drops LCC $0.63$ to $0.47$ and
RepoBench-P $0.57$ to $0.46$).

\newpage

\bibliographystyle{IEEEtran}
\bibliography{sample}

\end{document}